\pdfoutput=1
\documentclass[11pt]{article} 
\def\arXiv{1} 

\input{macros.sty}

\notarxiv{
	\usepackage[subtle, all=normal, mathspacing=tight, floats=tight, 
	sections=tight]{savetrees}
}

\title{Efficiently Solving MDPs with Stochastic Mirror Descent}

\arxiv{
	
\author{
Yujia Jin \\
Stanford University \\
\texttt{\href{mailto:yujiajin@stanford.edu}{yujiajin@stanford.edu}}
\and
	Aaron Sidford \\
Stanford University \\
\texttt{\href{mailto:sidford@stanford.edu}{sidford@stanford.edu}}
}

\date{}
}

\begin{document}

\maketitle

\begin{abstract}
We present a unified framework based on primal-dual stochastic mirror descent for approximately solving infinite-horizon Markov decision processes (MDPs) given a generative model. When applied to an average-reward MDP with $\A$ total state-action pairs and mixing time bound $\tmix$ our method computes an $\epsilon$-optimal policy with an expected $\runtimeAMDP$ samples from the state-transition matrix, removing the ergodicity dependence of prior art.  When applied to a $\gamma$-discounted MDP with $\A$ total state-action pairs our method computes an  $\epsilon$-optimal policy with an expected $\runtimeDMDP$ samples, matching the previous state-of-the-art up to a $(1-\gamma)^{-1}$ factor. Both methods are model-free, update state values and policies simultaneously, and run in time linear in the number of samples taken. We achieve these results through a more general stochastic mirror descent framework for solving bilinear saddle-point problems with simplex and box domains and we demonstrate the flexibility of this framework by providing further applications to constrained MDPs.
\end{abstract}

\newpage
\tableofcontents
\newpage

\section{Introduction}

Markov decision processes (MDPs) are a fundamental mathematical abstraction for sequential decision making under uncertainty and they serve as a basic modeling tool in reinforcement learning (RL) and stochastic control~\citep{bertsekas1995neuro,puterman2014markov,sutton2018reinforcement}. Two prominent classes of MDPs are average-reward MDPs (AMDPs) and discounted MDPs (DMDPs). Each have been studied extensively; AMDPs are applicable to optimal control, learning automata, and various real-world reinforcement learning settings ~\citep{mahadevan1996average,auer2007logarithmic,ouyang2017learning} and DMDPs have a number of nice theoretical properties including reward convergence and operator monotonicity~\citep{Bertsekas19}.

In this paper we consider the prevalent computational learning problem of finding an approximately optimal policy of an MDP given only restricted access to the model. In particular, we consider the problem of computing an $\eps$-optimal policy, i.e. a policy with an  additive $\eps$ error in expected cumulative reward over infinite horizon, under the standard assumption of a generative model~\citep{kearns99,K03}, which allows one to sample from state-transitions given the current state-action pair. This problem is well-studied and there are multiple known upper and lower bounds on its sample complexity~\cite{AMK12,W17d,SWWYY18,W19}.
In this work, we provide a unified framework based on primal-dual stochastic mirror descent (SMD) for learning an $\eps$-optimal policies for both AMDPs and DMDPs with a generative model. We show that this framework achieves sublinear running times for solving dense bilinear saddle-point problems with simplex and box domains, and (as a special case)  $\ell_\infty$ regression~\cite{sherman2017area,sidford2018coordinate}. As far as we are aware, this is the first such sub-linear running time for this problem. We achieve our results by applying this framework to saddle-point representations of AMDPs and DMDPs and proving that approximate equilibria yield approximately optimal policies.

Our MDP algorithms have sample complexity linear in the total number of state-action pairs, denoted by $\A$. For an AMDP with bounded mixing time $\tmix$ 
for all policies, we prove a sample complexity of $\runtimeAMDP$~\footnote{Throughout the paper we use $\widetilde{O}$ to hide poly-logarithmic factors in $\A$, $\tmix$, $1/(1-\gamma)$, $1/\epsilon$, and the number of states of the MDP.}, which removes the ergodicity condition of prior art~\citep{W17m} (which can in the worst-case be unbounded). For DMDP with discount factor $\gamma$, we prove a sample complexity of $\runtimeDMDP$, matching the best-known sample complexity achieved by primal-dual methods~\citep{cheng2020reduction} up to logarithmic factors, and matching the state-of-the-art~\citep{SWWYY18,W19} and lower bound~\citep{AMK12} up to a $(1-\gamma)^{-1}$ factor. 

We hope our method serves as a building block towards a more unified understanding the complexity of MDPs and RL. By providing a general SMD-based framework which is provably efficient for solving multiple prominent classes of MDPs we hope this paper may lead to a better understanding and broader application of the traditional convex optimization toolkit to modern RL. As a preliminary demonstration of flexibility of our framework, we show that it extends to yield new results for approximately optimizing constrained MDPs and hope it may find further utility.

\subsection{Problem Setup}

Throughout the paper we denote an MDP instance by a tuple $\calM\defeq(\calS,\calA,\calP,\RR,\gamma)$ with components defined as follows:
\begin{itemize}
\item 	 $\calS$ - a finite set of states where each  $i\in\calS$ is called a \emph{state of the MDP}, in tradition this is also denoted as $s$. 
\item    $\calA=\cup_{i\in[S]} \calA_i$ - a finite set of actions that is a collection of sets of actions  $\calA_i$ for states $i\in\calS$. We overload notation slightly and let $(i,a_i)\in\calA$ denote \emph{an action $a_i$ at state $i$}.  $\A\defeq |\calA| \defeq \sum_{i\in\calS}|\calA_i|$ denotes the total number of state-action pairs.%
\item $\calP$ - the collection of state-to-state transition probabilities where $\calP\defeq\{p_{ij}(a_i)|i,j\in\calS,a_i\in \calA_i\}$ and $p_{ij}(a_i)$ denotes the probability of transition to state $j$ when taking action $a_i$ at state $i$. 
\item $\RR$ - the vector of state-action transitional rewards where $\RR\in[0,1]^\calA,\ r_{i,a_i}$ is the instant reward received when taking action $a_i$ at state $i\in\calS$.\arxiv{\footnote{The assumption that $\rr$ only depends on state action pair $i,a_i$ is a common practice~\cite{SWWY18,SWWYY18}. Under a model with $r_{i,a_i,j}\in[0,1]$, one can use a straightforward reduction to consider the model with $r_{i,a_i}=\sum_{j\in\calS}p_{ij}(a_i)r_{i,a_i,j}$ by using $\widetilde{O}(\eps^{-2})$ ($\widetilde{O}((1-\gamma)^{-2}\eps^{-2})$) samples to estimate the expected reward given each state-action pair within $\eps/2$ ($(1-\gamma)\eps/2$) additive accuracy for mixing AMDP (DMDP), and finding an expected $\eps/2$-optimal policy of the new MDP constructed using those estimates of rewards. This will provably give an expected $\eps$-optimal policy for the original MDP.}}\notarxiv{\footnote{The assumption that $\rr$ only depends on state action pair $i,a_i$ is a common practice~\cite{SWWYY18}.}} %
\item $\gamma$ - the discount factor of MDP, by which one down-weights the reward in the next future step. When $\gamma\in(0,1)$, we call the instance a \emph{discounted MDP} (DMDP) and when $\gamma=1$, we call the instance an \emph{average-reward MDP} (AMDP).
\end{itemize}

We use $\PP\in\R^{\calA\times\calS}$ as the state-transition matrix where its $(i,a_i)$-th row corresponds to the transition probability from state $i\in\calS$ where $a_i\in\calA_i$ to state $j$. Correspondingly we use $\hat{\II}$ as the matrix with $a_i$-th row corresponding to $\ee_i$, for all $i\in\calS,a_i\in\calA_i$.

Now, the model operates as follows: when at state $i$, one can pick an action $a_i$ from the given action set $\calA_i$. This generates a reward $r_{i,a_i}$. Also based on the transition model with probability $p_{ij}(a_i)$, it transits to state $j$ and the process repeats. 
Our goal is to compute a random policy which determines which actions to take at each state. 
A random policy  %
 is a collection of probability distributions $\pi\defeq\{\pi_i\}_{i\in\calS}$, where $\pi_i\in \Delta^{\calA_i}$ %
 is a vector in the $|\mathcal{A}_i|$-dimensional simplex with $\pi_i(a_i)$ denoting the probability of taking $a_i\in\calA_i$ at action $j$. One can extend $\pi_i$ to the set of $\Delta^{\calA}$ by filling in $0$s on entries corresponding to other states $j\neq i$, and denote $\Pi\in\R^{\calS\times \calA}$ as the concatenated policy matrix with $i$-th row being the extended $\Delta_i$. We denote $\PP^{\pi}$ as the trasitional probability matrix of the MDP when using policy $\pi$, thus we have $\PP^{\pi}(i,j)\defeq\sum_{a_i\in\calA_i}\pi_i(a_i)p_{ij}(a_i)=\Pi\cdot\PP$ for all $i,j\in\calS$, where $\cdot$ in the right-hand side (RHS) denotes matrix-matrix multiplication. Further, we let $\rr^{\pi}$  denote  corresponding average reward under policy $\pi$ defined as $\rr^{\pi}\defeq\Pi\cdot\rr$, where $\cdot$ in RHS denotes matrix-vector multiplication.
 We use $\II$ to denote the standard identity matrix if computing with regards to probability transition matrix $\Pi^{\pi}\in\R^{\calS\times \calS}$.

Given an MDP instance $\calM\defeq(\calS,\calA,\calP,\RR,\gamma)$ and an initial distribution over states $\qq\in\Delta^{\calS}$, we are interested in finding the optimal $\pi^*$ among all policy $\pi$ that maximizes the following cumulative reward $\bar{v}^\pi$ of the MDP:
\arxiv{
\begin{align*}
\pi^*\defeq & \arg\max_{\pi}\bar{v}^\pi
\enspace\text{ where }\enspace \bar{v}^\pi\defeq 
\begin{cases}
	 \E^{\pi}\left[\sum\limits_{t=1}^\infty \gamma^{t-1}r_{i_t,a_t}|i_1\sim\qq\right] ,
	 & \forall\gamma\in(0,1) \enspace \text{i.e., DMDPs}\\
	 \lim\limits_{T\rightarrow\infty}\frac{1}{T}\E^{\pi}\left[\sum\limits_{t=1}^T r_{i_t,a_t}|i_1\sim\qq\right] ,
	 & \gamma=1 \enspace \text{i.e., AMDPs} ~.
\end{cases}
\end{align*}}
\notarxiv{
\begin{align*}
\pi^*\defeq & \arg\max_{\pi}\bar{v}^\pi
\enspace\text{ where }\\
\enspace \bar{v}^\pi\defeq & 
\begin{cases}
	 \E^{\pi}\left[\sum\limits_{t=1}^\infty \gamma^{t-1}r_{i_t,a_t}|i_1\sim\qq\right] ,
	 & \text{i.e., DMDPs}\\
	 \lim\limits_{T\rightarrow\infty}\frac{1}{T}\E^{\pi}\left[\sum\limits_{t=1}^T r_{i_t,a_t}|i_1\sim\qq\right] ,
	 & \text{i.e., AMDPs} ~.
\end{cases}
\end{align*}
}
Here $\{i_1,a_1,i_2,a_2,\cdots,i_t,a_t\}$ are state-action transitions generated by the MDP under policy $\pi$. For the DMDP case, it also holds by definition that $\bar{v}^\pi\defeq\qq^\top(\II-\gamma\PP^\pi)^{-1} \rr^\pi$.

For the AMDP case (i.e. when $\gamma=1$), we define $\nnu^\pi$ as the stationary distribution under policy $\pi$ satisfying $\nnu^\pi=(\PP^\pi)^\top\nnu^\pi$. To ensure the value of $\bar{v}^\pi$ is well-defined, we restrict our attention to a subgroup which we call \emph{mixing AMDP} satisfying the following mixing assumption:

\begin{assumption}
\label{assum}
	An AMDP instance is \emph{mixing} if $\tmix$, defined as follows, is bounded by $1/2$, i.e.
	\begin{equation*}\label{def-mixingtime}
	\tmix\defeq\max_{\pi}\left[\argmin_{t\ge1 }\max_{\qq\in\Delta^\calS} \lones{({\PP^{\pi}}^{\top})^t\qq-\nnu^\pi}\right]\le \tfrac{1}{2}.
	\end{equation*}
\end{assumption}

The mixing condition assumes for arbitrary policy $\pi$ and arbitrary initial state, the resulting Markov chain leads toward a distribution close enough to its stationary distribution $\nnu^\pi$ starting from any initial state $i$ in $O(\tmix)$ time steps. This assumption implies the the uniqueness of the stationary distribution, makes $\bar{v}^\pi$ above well-defined with the equivalent $\bar{v}^\pi=(\nnu^\pi)^\top\rr^\pi$, governing the complexity of our mixing AMDP algorithm. This assumption is key for the results we prove (Theorem~\ref{thm:sub-mixing-main}) and equivalent to the one in~\citet{W17m}, up to constant factors.

By nature of the definition of \emph{mixing AMDP}, %
we note that the value of a strategy $\pi$ is independent of initial distribution $\qq$ and only dependent of the eventual stationary distribution as long as the AMDP is mixing, which also implies $\bar{v}^\pi$ is always well-defined. For this reason, sometimes we also omit $i_1\sim\qq$ in the corresponding definition of $\bar{v}^\pi$. %

We call a policy $\pi$ an \emph{$\eps$-(approximate) optimal policy} for the MDP problem, if it satisfies
$\bar{v}^{\pi}\ge \bar{v}^{*}-\eps$.\footnote{Hereinafter, we use superscript $^*$ and $^{\pi^*}$ interchangeably.} We call a policy an expected $\eps$-(approximate) optimal policy if it satisfies the condition in expectation, i.e. %
$\E\bar{v}^{\pi}\ge \bar{v}^{*}-\eps$.
The goal of paper is to develop efficient algorithms that find (expected) $\eps$-optimal policy for the given MDP instance assuming access to a generative model. %

\subsection{Main Results}

The main result of the paper is a unified framework based on randomized primal-dual stochastic mirror descent (SMD) that with high probability finds an (expected) $\eps$-optimal policy with some sample complexity guarantee. Formally we provide two algorithms (see Algorithm~\ref{alg:sublinear-mixing} for both cases) with the following guarantees respectively.

\begin{restatable}{theorem}{restatemixingmain}
\label{thm:sub-mixing-main}
Given a mixing AMDP tuple $\calM=(\calS,\calA,\calP,\RR)$, let $\epsilon \in (0,1)$, one can construct an expected $\eps$-optimal policy $\pi^\eps$ from the decomposition (see Section~\ref{ssec:sub-mixing}) of output $\mmu^\eps$ of Algorithm~\ref{alg:sublinear-mixing} %
with sample complexity \runtimefullmb.
\end{restatable}

\begin{restatable}{theorem}{restatediscountedmain}
\label{thm:sub-discounted-main}
Given a DMDP tuple $\calM=(\calS,\calA,\calP,\RR,\gamma)$ with discount factor $\gamma\in(0,1)$, let $\epsilon \in (0,1)$, one can construct an expected $\eps$-optimal policy $\pi^\eps$ from the decomposition (see Section~\ref{ssec:sub-discounted}) of output $\mmu^\eps$ of Algorithm~\ref{alg:sublinear-mixing} %
with sample complexity \runtimefulldb.
\end{restatable}

We remark that for both problems, the algorithm also gives with high probability an $\eps$-optimal policy at the cost of an extra $\log(1/\delta)$ factor to the sample complexity through a reduction from high-probability to expected optimal policy (see~\citet{W17m} for more details). Note that we only obtain randomized policies, and we leave the question of getting directly deterministic policies as an interesting open direction. %

\begin{algorithm}
\caption{SMD for mixing AMDP / DMDPs}
	\label{alg:sublinear-mixing}
\begin{algorithmic}[1]
 \STATE \textbf{Input:} MDP tuple $\calM=(\calS,\calA,\calP,\RR,\gamma)$, initial $(\vv_0,\mmu_0)\in\calV\times\calU$, with $\calV\defeq 2M\cdot[-1,1]^\calS$. %
 \STATE	\textbf{Output:} An expected $\eps$-approximate solution $(\vv^\eps,\mmu^\eps)$ for problem~\eqref{def-minimax-mixing}.
\STATE \textbf{Parameter:} Step-size $\eta\vsf$, $\eta\musf$, number of iterations $T$, accuracy level $\eps$.
		\FOR{$t=1,\ldots,T$}
			\STATE \algorithmiccomment{\emph{$\vv$ gradient estimation}}
			\STATE Sample $(i,a_i)\sim [\mmu]_{i,a_i}$, $j\sim p_{ij}(a_i)$, $i'\sim q_{i'}$
			\STATE Set $\tilde{g}\vsf_{t-1} = \begin{cases} \ee_j-\ee_i & \text{mixing}\\ (1-\gamma)\ee_{i'} + \gamma\ee_j-\ee_i & \text{discounted}\end{cases}$
			
			\vspace{0.05in}
			\STATE \algorithmiccomment{\emph{$\mmu$ gradient estimation}}
			\STATE Sample $(i,a_i)\sim\frac{1}{\A}$, $j\sim p_{ij}(a_i)$
			\STATE Set $\tilde{g}\musf_{t-1} = \begin{cases} \A(v_i- v_j-r_{i,a_i})
			\ee_{i,a_i} & \text{mixing}\\ \A(v_i-\gamma v_j-r_{i,a_i})\ee_{i,a_i} & \text{discounted}\end{cases}$
			
			\vspace{0.05in}
			\STATE \algorithmiccomment{\emph{Stochastic mirror descent steps ($\Pi$ as projection)}}
			\STATE $\vv_{t} \leftarrow \Pi_{\B_{2M}^\calS}(\vv_{t-1}-\eta\vsf \tilde{g}_{t-1}\vsf)$
			\STATE $\mmu_{t} \leftarrow \Pi_{\Delta^\calA}(\mmu_{t-1}\circ\exp(-\eta\musf \tilde{g}\musf_{t-1}))$
			
		\ENDFOR
		\STATE \textbf{Return} $(\vv^\eps,\mmu^\eps)\leftarrow\frac{1}{T}\sum_{t\in[T]} (\vv_t,\mmu_t)$ 		%
\end{algorithmic}
\end{algorithm}

Table~\ref{tab:runtime} gives a comparison of sample complexity between our methods and prior methods\footnote{Most methods assume a uniform action set $\calA$ for each of the $|\calS|$ states, but can also be generalizedd to the  non-uniform case parameterized by $\A$.} for computing an $\eps$-approximate policy in DMDPs and AMDPs given a generative model. 

As a generalization, we show how to solve constrained average-reward MDPs (cf.~\cite{altman1999constrained}, a generalization of average-reward MDP) using the primal-dual stochastic mirror descent framework in Section~\ref{sec:constrained}. We build an algorithm that solves the constrained problem~\eqref{def:conAMDP} to $\eps$-accuracy within sample complexity $O({(\tmix^2\A+K)D^2}{\eps^{-2}}\log(K\A))$, where $K$ and $D^2$ are number and size of the constraints. To the best of our knowledge this is the first sample complexity results for constrained MDPs given by a generative model.

As a byproduct, our framework in Section~\ref{sec:framework} also gives a stochastic algorithms (see Algorithm~\ref{alg:framework}) that find an expected $\eps$-approximate solution of $\ell_\infty$-$\ell_1$ bilinear minimax problems of the form\arxiv{
\[\min_{\xx\in[-1,1]^n}\max_{\yy\in\Delta^m}\yy^\top\M\xx+\bb^\top\xx-\cc^\top\yy\]}
\notarxiv{
$\min_{\xx\in[-1,1]^n}\max_{\yy\in\Delta^m}\yy^\top\M\xx+\bb^\top\xx-\cc^\top\yy$
}
to $\eps$-additive accuracy with runtime $\Otil{((m+n)\linf{\M}^2+n\lones{\bb}^2+m\linf{\cc}^2)\eps^{-2}}$ given $\ell_1$ sampler of iterate $\yy$, and $\ell_1$ samplers based on the input entries of $\M$, $\bb$ and $\cc$ (see Corollary~\ref{cor:linf-reg} for details), where we define $\linf{\M}\defeq\max_{i}\lones{\M(i,:)}$. Consequently, it solves (box constrained) $\ell_\infty$ regression problems of form
\arxiv{
\begin{equation}\label{def-linf-reg}
\min_{\xx\in[-1,1]^n}\linf{\M\xx-\cc}
\end{equation}
}
\notarxiv{$\min_{\xx\in[-1,1]^n}\linf{\M\xx-\cc}$}
to $\eps$-additive accuracy within runtime $\Otil{((m+n)\linf{\M}^2+m\linf{\cc}^2)\eps^{-2}}$ given similar sampling access (see Remark~\ref{rem:linf-reg} for details and Table~\ref{tab:runtime-linf-reg} for comparison with previous results).
\arxiv{
\begin{table*}[t]
	\centering
	\renewcommand{\arraystretch}{1.4}
	\bgroup
	\everymath{\displaystyle}
	\begin{tabular}{c c c}
		\toprule
		Type                        & Method                    & Sample Complexity\\ 
		\midrule
		\multirow{2}{*}{\makecell{mixing AMDP}}     
		&   Primal-Dual Method~\citep{W17m}  &   $\Otilb{\tau^2\tmix^2\A\eps^{-2}}$
		  \\ 
		  & \textbf{Our method} (Theorem~\ref{thm:sub-mixing-main})                &    $\Otilb{\tmix^2\A\eps^{-2}}$
		  \\ 
		\hline
		\multirow{9}{*}{\makecell{DMDP}}      & Empirical QVI~\citep{AMK12}
	 &  $\Otilb{(1-\gamma)^{-5}\A\eps^{-2}}$      \\ 
	  & Empirical QVI~\citep{AMK12}
	 &  $\Otilb{(1-\gamma)^{-3}\A\eps^{-2}}$, $\eps=\widetilde{O}(\tfrac{1}{\sqrt{(1-\gamma)|\calS|}})$      \\ 
		&   Primal-Dual Method~\citep{W17d}
	 &    $\Otilb{(1-\gamma)^{-6}|\calS|^2\A\eps^{-2}}$  \\ 
		&   Primal-Dual Method~\citep{W17d}  &   $\Otilb{\tau^4(1-\gamma)^{-4}\A\eps^{-2}}$  \\ 
		&   Online Learning Method~\citep{cheng2020reduction}  &   $\Otilb{(1-\gamma)^{-4}\A\eps^{-2}}$  \\ 
		&   Variance-reduced Value Iteration~\citep{SWWY18}      &  $\Otilb{(1-\gamma)^{-4}\A\eps^{-2}}$   \\ 
		&   Variance-reduced QVI~\citep{SWWYY18}      &  $\Otilb{(1-\gamma)^{-3}\A\eps^{-2}}$   \\ 
		&   Empirical MDP + Blackbox~\citep{agarwal2019optimality}      &  $\Otilb{(1-\gamma)^{-3}\A\eps^{-2}}$   \\ 
		&   Variance-reduced Q-learning~\citep{W19}      &  $\Otilb{(1-\gamma)^{-3}\A\eps^{-2}}$   \\ 
		& \textbf{Our method} (Theorem~\ref{thm:sub-discounted-main})               &  $\Otilb{(1-\gamma)^{-4}\A\eps^{-2}}$    \\ 
		\bottomrule
	\end{tabular}
	\caption{\textbf{Comparison of sample complexity to get $\eps$-optimal policy among stochastic methods.} Here $\calS$ denotes state space, $\A$ denotes number of state-action pair, $\tmix$ is mixing time for mixing AMDP, and $\gamma$ is discount factor for DMDP. Parameter $\tau$ shows up whenever the designed algorithm requires additional ergodic condition for MDP, i.e. there exists some distribution $\qq$ and $\tau>0$ satisfying $\sqrt{1/\tau}\qq\le \nnu^\pi\le\sqrt{\tau}\qq$, $\forall$ policy $\pi$ and its induced stationary distribution $\nnu^\pi$. }\label{tab:runtime} 
	\egroup
\end{table*}}
\notarxiv{
\begin{table*}[t]
	\centering
	\renewcommand{\arraystretch}{1.4}
	\bgroup
	\everymath{\displaystyle}
	\begin{tabular}{c c c}
		\toprule
		Type                        & Method                    & Sample Complexity\\ 
		\midrule
		\multirow{2}{*}{\makecell{mixing AMDP}}     
		&   Primal-Dual Method~\citep{W17m}  &   $\Otilb{\tau^2\tmix^2\A\eps^{-2}}$
		  \\ 
		  & \textbf{Our method} (Theorem~\ref{thm:sub-mixing-main})                &    $\Otilb{\tmix^2\A\eps^{-2}}$
		  \\ 
		\hline
		\multirow{9}{*}{\makecell{DMDP}}      & Empirical QVI~\citep{AMK12}
	 &  $\Otilb{(1-\gamma)^{-5}\A\eps^{-2}}$      \\ 
	  		&   Primal-Dual Method~\citep{W17d}
	 &    $\Otilb{(1-\gamma)^{-6}|\calS|^2\A\eps^{-2}}$  \\ 
		&   Primal-Dual Method~\citep{W17d}  &   $\Otilb{\tau^4(1-\gamma)^{-4}\A\eps^{-2}}$  \\ 
				&   Variance-reduced QVI~\citep{SWWYY18}      &  $\Otilb{(1-\gamma)^{-3}\A\eps^{-2}}$   \\ 
		&   Empirical MDP + Blackbox~\citep{agarwal2019optimality}      &  $\Otilb{(1-\gamma)^{-3}\A\eps^{-2}}$   \\ 
		&   Variance-reduced Q-learning~\citep{W19}      &  $\Otilb{(1-\gamma)^{-3}\A\eps^{-2}}$   \\ 
		& \textbf{Our method} (Theorem~\ref{thm:sub-discounted-main})               &  $\Otilb{(1-\gamma)^{-4}\A\eps^{-2}}$    \\ 
		\bottomrule
	\end{tabular}
	\caption{\textbf{Comparison of sample complexity to get $\eps$-optimal policy among stochastic methods (complete version see Appendix~\ref{app:table-com}).} Here $\calS$ denotes state space, $\A$ denotes number of state-action pair, $\tmix$ is mixing time for mixing AMDP, and $\gamma$ is discount factor for DMDP. Parameter $\tau$ shows up whenever the designed algorithm requires additional ergodic condition for MDP, i.e. there exists some distribution $\qq$ and $\tau>0$ satisfying $\sqrt{1/\tau}\qq\le \nnu^\pi\le\sqrt{\tau}\qq$, $\forall$ policy $\pi$ and its induced stationary distribution $\nnu^\pi$. }\label{tab:runtime} 
	\egroup
\end{table*}
}

\subsection{Technique Overview}

We adopt the idea of formulating the MDP problem as a bilinear saddle point problem in light of linear duality, following the line of randomized model-free primal-dual $\pi$ learning studied in \citet{W17d,W17m}. %
This formulation relates MDP to solving bilinear saddle point problems with box and simplex domains, which falls into well-studied generalizations of convex optimization~\citep{nemirovski2004prox,CJST19}.

We study the efficiency of standard stochastic mirror descent (SMD) for this bilinear saddle point problem where the minimization (primal) variables are constrained to a rescaled box domain and the maximization (dual) variables are constrained to the simplex. We use the idea of local-norm variance bounds emerging in~\citet{Shalev12,CJST19,CJST20} to design and analyze efficient stochastic estimators for the gradient of this problem that have low-variance under the corresponding local norms. We provide a new analytical way to bound the quality of an approximately-optimal policy constructed from the approximately optimal solution of bilinear saddle point problem, which utilizes the influence of  the dual constraints under minimax optimality. Compared with prior work, by extending the primal space by a constant size and providing new analysis, we eliminate ergodicity assumptions made in prior work for mixing AMDPs.  Combining these pieces, we obtain a natural SMD algorithm which solves mixing AMDPs (DMDPs) as stated in Theorem \ref{thm:sub-mixing-main} (Theorem \ref{thm:sub-discounted-main}).

\subsection{Related Work}

\subsubsection{On Solving MDPs}

Within the tremendous body of study on MDPs, and more generally reinforcement learning, stands the well-studied classic problem of computational efficiency (i.e. iteration number, runtime, etc.) of finding optimal policy, given the entire MDP instance as an input. Traditional deterministic methods for the problems are value iteration, policy iteration, and linear programming.~\citep{Bertsekas19,Ye11}, which find an approximately optimal policy to high-accuracy but have superlinear runtime in the usually high problem dimension $\Omega(|\calS|\cdot\A)$. %

To avoid the necessity of knowing the entire problem instance and having superlinear runtime dependence, more recently, researchers have designed stochastic algorithms assuming only a generative model that samples from state-transitions~\citep{K03}. \citet{AMK12} proved a lower bound of $\Omega((1-\gamma)^{-3}\A\eps^{-2})$ while also giving a Q-value-iteration algorithm with a higher guaranteed sample complexity. \citet{W17d} designed a randomized primal-dual method, an instance of SMD with slightly different sampling distribution and updates for the estimators, which obtained sublinear sample complexity for the problem provided certain ergodicity assumptions were made.
The sample complexity upper bound was improved (without an ergodicity assumptions) in \citet{SWWY18} using variance-reduction ideas, and was further improved to match (up to logarithmic factors) lower bound in \cite{SWWYY18} using a type of $Q$-function based variance-reduced randomized value iteration. Soon after in \citet{W19}, a variance-reduced $Q$-learning method also achieved nearly tight sample complexity for the discounted case and in~\citet{agarwal2019optimality} the authors used a different approach, solving an empirical MDP, that shows $\widetilde{O}((1-\gamma)^{-3}\A\eps^{-2})$ samples suffice. %

While several methods match (up to logarithmic factors) the lower bound shown for sample complexity for solving DMDP~\citep{SWWYY18,W19}, it is unclear whether one can design similar methods for AMDPs and obtain optimal sample complexities. The only related work for sublinear runtimes for AMDPs uses primal-dual $\pi$-learning~\cite{W17m}, following the stochastic primal-dual method in~\cite{W17d}. This method is also a variant of SMD methods and compared to our algorithm, theirs has a different domain setup, different update forms, and perhaps, a more specialized analysis. The sample complexity obtained by \cite{W17m} is $\Otil{\tau^2\tmix^2\A\eps^{-2}}$, which (as in the case of DMDPs) depends polynomially on the ergodicity parameter $\tau > 0$, and can be arbitrarily large in general.

\arxiv{
}
Whether randomized primal-dual SMD methods necessarily incur much higher computational cost when solving DMDPs and necessarily depend on ergodicity when solving both DMDPs and AMDPs is a key motivation of our work. Obtaining improved primal-dual SMD methods for solving MDPs creates the possibility of leveraging the flexibility of optimization machinery to easily obtain improved sample complexities in new settings easily (as our constrained MDP result highlights).

Independently, \citep{cheng2020reduction} made substantial progress on clarifying the power of primal-dual methods for solving MDPs by providing a
 $\widetilde{O}((1-\gamma)^{-4}\A\eps^{-2})$ sample complexity for solving DMDPs (with no ergodicity dependence), using an online learning regret analysis.\footnote{We were unaware of this recent result until the final preparation of this manuscript.}
In comparison, we offer a general framework which also applies to the setting of mixing AMDs to achieve the state-of-the-art sample complexity bounds for mixing AMDPs, and extend our framework to solving constrained AMDPs; \citep{cheng2020reduction}  connects with the value of policies with regret in online learning more broadly, and offers extensions to DMDPs with linear approximation. It would be interesting to compare the techniques and see if all the results of each paper are achievable through the techniques of the other.
    
  Table~\ref{tab:runtime} includes a complete comparison between our results and the prior art for discounted MDP and mixing AMDP.

\subsubsection{On $\ell_\infty$ Regression and Bilinear Saddle Point Problem}

Our framework gives a stochastic method for solving $\ell_\infty$ regression, which is a core problem in both combinatorics and continuous optimization due to its connection with maximum flow and  linear programming~\citep{lee2014path,lee2015efficient}. Classic methods build on solving a smooth approximations of the problem~\citep{nesterov2005smooth} or finding the right regularizers and algorithms for its correspondingly primal-dual minimax problem~\citep{nemirovski2004prox,nesterov2007dual}. These methods have recently been improved to $\Otil{\nnz\linf{\M}\eps^{-1}}$ using a joint regularizer with nice area-convexity properties in~\citet{sherman2017area} or using accelerated coordinate method with a matching runtime bound in sparse-column case in \citet{sidford2018coordinate} .

\notarxiv{ 

In comparison to all the state-of-the-art, for dense input data matrix our method gives the first algorithm with sublinear runtime dependence $O(m+n)$ instead of $O(\nnz)$. For completeness we include a comparison of runtimes for methods mentioned above in  Appendix~\ref{app:table} (see Table~\ref{tab:runtime-linf-reg}).

}

\arxiv{

In comparison to all the state-of-the-art, for dense input data matrix our method gives the first algorithm with sublinear runtime dependence $O(m+n)$ instead of $O(\nnz)$. For completeness here we include Table~\ref{tab:runtime-linf-reg} that make comparisons between our sublinear $\ell_\infty$ regression solver and prior art. We remark for dense matrix $\M$, our method is the only sublinear method along this line of work for approximately solving $\ell_\infty$ regression problem.

\begin{table*}[htb]
	\centering
	\renewcommand{\arraystretch}{1.4}
	\bgroup
	\everymath{\displaystyle}
	\begin{tabular}{c c}
		\toprule
		Method                    & Runtime\\ 
		\midrule
	Smooth Approximation~\citep{nesterov2005smooth}
	 &  $\Otil{\nnz\linf{\M}^2\eps^{-2}}$\quad or\quad $\Otil{\nnz\sqrt{n}\linf{\M}\eps^{-1}}$      \\ 
		  Mirror-prox Method~\citep{nemirovski2004prox}
	 &    $\Otil{\nnz\linf{\M}^2\eps^{-2}}$  \\ 
		  Dual Extrapolation~\citep{nesterov2007dual}  &   $\Otil{\nnz\linf{\M}^2\eps^{-2}}$  \\ 
		  Dual Extrapolation with Joint Regularizer~\citep{sherman2017area}      &  $\Otil{\nnz\linf{\M}\eps^{-1}}$   \\ 
		  Accelerated Coordinate Method~\citep{sidford2018coordinate}      &  $\Otilb{nd^{2.5}\linf{\M}\eps^{-1}}$   \\ 
		\textbf{Our method} (Remark~\ref{rem:linf-reg})               &  $\Otilb{(m+n)\linf{\M}^2\eps^{-2}}$    \\ 
		\bottomrule
	\end{tabular}
	\caption{\textbf{Runtime Comparison of $\eps$-approximate $\ell_\infty$-regression Methods:} For simplicity, here we only state for the simplified problem, $\min_{\xx\in\ball^n}\linf{\M\xx}$, where $\M\in\R^{m\times n}$ with $\nnz$ nonzero entries and $d$-sparse columns.}\label{tab:runtime-linf-reg}
	\egroup
\end{table*}

}
Our sublinear method for $\ell_\infty$-regression is closely related to a line of work on  obtaining efficient stochastic methods for  approximately solving \emph{matrix games}, i.e. bilinear saddle point problems~\citep{grigoriadis1995sublinear,clarkson2012sublinear,palaniappan2016stochastic}, and, in particular, a recent line of work by the authors and collaborators~\citep{CJST19,CJST20} that explores the benefit of careful sampling and variance reduction in matrix games. In~\citet{CJST19} we provide a framework to analyze variance-reduced SMD under local norms to obtain better complexity bounds for different domain setups, i.e. $\ell_1$-$\ell_1$, $\ell_1$-$\ell_2$, and $\ell_2$-$\ell_2$ where $\ell_1$ corresponds to the simplex and $\ell_2$ corresponds to the Euclidean ball. In~\citet{CJST20} we study the improved sublinear and variance-reduced coordinate methods for these domain setups utilizing the desgn of optimal gradient estimators.  This paper adapts the local norm analysis and coordinate-wise gradient estimator design in~\citet{CJST19,CJST20} to obtain our SMD algorithm and analysis for $\ell_1$-$\ell_\infty$ games.

\section{Preliminaries}
\label{sec:prelim}

First, we introduce several known tools for studying MDPs.

\subsection{Bellman Equation.}
\label{ssec:pre-BE}

For mixing AMDP,  $\bar{v}^*$ is the optimal average reward if and only if there exists a vector $\vv^*=(v^*_i)_{i\in\calS}$ satisfying its corresponding \emph{Bellman equation}~\citep{Bertsekas19} 
\begin{equation}
\label{def-Bellman-mixing}	
\bar{v}^*+v_i^*=\max_{a_i\in\calA_i}  \left\{ \sum_{j\in\calS} p_{ij}(a_i)v^*_j+r_{i,a_i} \right\} ,\forall i\in\calS.
\end{equation}
When considering a mixing AMDP as in the paper, the existence of solution to the above equation
can be guaranteed. However, it is important to note that one cannot guarantee the uniqueness of the optimal $\vv^*$. In fact, for each optimal solution $\vv^*$, $\vv^*+c\1$ is also an optimal solution. 

For DMDP, one can show that at optimal policy $\pi^*$, each state $i\in\calS$ can be assigned an optimal cost-to-go value $v^*_i$ satisfying the following \emph{Bellman equation}~\citep{Bertsekas19}
\begin{equation}
\label{def-Bellman-discounted}	
v_i^*=\max_{a_i\in\calA_i} \left\{ \sum_{j\in\calS}\gamma p_{ij}(a_i)v^*_j+r_{i,a_i} \right\},\forall i\in\calS.
\end{equation}
When $\gamma\in(0,1)$, it is straightforward to guarantee the existence and uniqueness of the optimal solution $\vv^*\defeq(v^*_i)_{i\in\calS}$ to the system.

\subsection{Linear Programming (LP) Formulation.}
\label{ssec:pre-lp}
We can further write the above Bellman equations equivalently as the following primal or dual linear programming problems. We define the domain as $\B^\calS_m\defeq m\cdot[-1,1]^S$ where $\B$ stands for box, and $\Delta^n\defeq\{\Delta\in\R^n,\Delta_i\ge0,\sum_{i\in[n]}\Delta_i=1\}$ for standard $n$-dimension simplex.

For mixing AMDP case, the linear programming formulation leveraging matrix notation is  (with $(P)$, $(D)$ representing (equivalently) the primal form and the dual form respectively)
\arxiv{

\noindent\begin{minipage}{.5\linewidth}
\begin{equation*}
  \begin{aligned}
    & \text{(P)} & \min_{\bar{v},\vv} & & \bar{v} &    \\
&  & \text{subject to } & &   \bar{v} \cdot \1+&(\hat{\II}-\PP) \vv-\rr\ge0,
  \end{aligned}
\end{equation*}
\end{minipage}%
\begin{minipage}{.5\linewidth}
\vspace{1em}
\begin{equation}
\label{def-lp-mixing-matrix}
  \begin{aligned}
    & \text{(D)} & \max_{\mmu\in\Delta^\calA} & & \mmu^\top \rr &\\
&  & \text{subject to } & & (\hat{\II}- \PP)^\top\mmu & = \0.
  \end{aligned}
\end{equation}
\end{minipage}

}
\notarxiv{
\begin{equation}
\label{def-lp-mixing-matrix}	
  \begin{aligned}
    & \text{(P)} & \min_{\bar{v},\vv} & & \bar{v} &    \\
&  & \text{subject to } & &   \bar{v} \cdot \1+&(\hat{\II}-\PP) \vv-\rr\ge0,\\
    & \text{(D)} & \max_{\mmu\in\Delta^\calA} & & \mmu^\top \rr &\\
&  & \text{subject to } & & (\hat{\II}- \PP)^\top\mmu & = \0.
  \end{aligned}
  \end{equation}
}
The optimal values of both systems are the optimal expected cumulative reward $\bar{v}^*$ under optimal policy $\pi^*$, thus hereinafter we use $\bar{v}^*$ and $\bar{v}^{\pi^*}$ interchangeably. Given the optimal dual solution $\mmu^*$, one can without loss of generality impose the constraint of $\langle\II^\top\mmu^*,\vv^*\rangle=0$~\footnote{$\hat{\II}^\top\mmu^*$  represents the stationary distribution over states given optimal policy $\pi^*$ constructed from optimal dual variable $\mmu^*$.} to ensure uniqueness of the primal problem (P).

For DMDP case, the equivalent linear programming is
\arxiv{

\noindent\begin{minipage}{.5\linewidth}
\vspace{1em}
\begin{equation*}
  \begin{aligned}
   & \text{(P)} & \min_{\vv} & & (1-\gamma)\qq^\top\vv &    \\
&  & \text{subject to } & &   (\hat{\II}-\gamma\PP) \vv-\rr & \ge0,
  \end{aligned}
\end{equation*}
\end{minipage}%
\begin{minipage}{.5\linewidth}
\begin{equation}
\label{def-lp-discounted-matrix}
  \begin{aligned}
& \text{(D)} & \max_{\mmu\in\Delta^\calA} & & \mmu^\top \rr &\\
&  & \text{subject to } & & (\hat{\II}- \gamma\PP)^\top\mmu & = (1-\gamma)\qq.
  \end{aligned}
\end{equation}
\end{minipage}

}
\notarxiv{
\begin{equation}
\label{def-lp-discounted-matrix}
\begin{aligned}
& \text{(P)} & \min_{\vv\in\B_{2M}^S} & & (1-\gamma)\qq^\top\vv &    \\
&  & \text{subject to } & &   (\hat{\II}-\gamma\PP) \vv-\rr & \ge0, \\
& \text{(D)} & \max_{\mmu\in\Delta^\calA} & & \mmu^\top \rr &\\
&  & \text{subject to } & & (\hat{\II}- \gamma\PP)^\top\mmu & = (1-\gamma)\qq.
\end{aligned}
\end{equation}
}
Given a fixed initial distribution $\qq$, the optimal values of both systems are a $(1-\gamma)$ factor of the optimal expected cumulative reward , i.e. $(1-\gamma)\bar{v}^*$ under optimal policy $\pi^*$.

\subsection{Minimax Formulation.}
\label{ssec:pre-minimax}

By standard linear duality, we can recast the problem formulation in Section~\ref{ssec:pre-lp} using the method of Lagrangian multipliers, as bilinear saddle-point (minimax) problems. For AMDPs the minimax formulation is
\arxiv{
\begin{align}
& \min_{\bar{v},\vv\in\B^\calS_{2M}}\max_{\mmu\in\Delta^\calA} f(\bar{v},\vv,\mmu),\label{def-minimax-mixing}\\
& \text{ where } f(\bar{v},\vv,\mmu)  \defeq \bar{v}+\mmu^\top(-\bar{v}\cdot \1+(\PP-\hat{\II})\vv+\rr)\nonumber = \mmu^\top((\PP-\hat{\II})\vv+\rr)\nonumber
\end{align}	
}
\notarxiv{
\begin{align}
& \min_{\bar{v},\vv\in\B^\calS_{2M}}\max_{\mmu\in\Delta^\calA} & & f(\bar{v},\vv,\mmu),\label{def-minimax-mixing}\\
& \text{ where } f(\bar{v},\vv,\mmu)  & \defeq &  \bar{v}+\mmu^\top(-\bar{v}\cdot \1+(\PP-\hat{\II})\vv+\rr)\nonumber\\
& & = & \mmu^\top((\PP-\hat{\II})\vv+\rr)\nonumber
\end{align}	}
For DMDPs the minimax formulation is
\begin{align}
& \min_{\vv\in\B^\calS_{2M}}\max_{\mmu\in\Delta^\calA} f_\qq(\vv,\mmu),\label{def-minimax-discounted}\\
& \text{ where }f_\qq(\vv,\mmu) \defeq (1-\gamma)\qq^\top \vv+\mmu^\top((\gamma \PP-\hat{\II})\vv+\rr).\nonumber
\end{align}	
Note in both cases we have added the constriant of $\vv\in\ball^{\calS}_{2M}$. The $M$ is different for each case, and will be specified in Section~\ref{ssec:bound-mixing} and~\ref{ssec:bound-discounted} to ensure that $\vv^*\in\ball_M^{\calS}$. As a result, constraining the bilinear saddle point problem on a restriced domain for primal variables will not affect the optimality of the original optimal solution due to it global optimality, but will considerably save work for the algorithm by considering a smaller domain. Besides we are also considering $\vv\in\ball^{\calS}_{2M}$ instead of $\vv\in\ball^{\calS}_{M}$ for solving MDPs to ensure a stricter optimality condition for the dual variables, see Lemma~\ref{lem:norm-bounds-mixing} for details.

For each problem we define the duality gap of the minimax problem $\min_{\vv\in\calV}\max_{\mmu\in\calU} f(\vv,\mmu)$ at a given pair of feasible solution $(\vv,\mmu)$ as
\arxiv{
\[
\Gap(\vv,\mmu)\defeq\max_{\mmu'\in\calU}f(\vv,\mmu’)-\min_{\vv'\in\calV}f(\vv’,\mmu).
\]}
\notarxiv{$\Gap(\vv,\mmu)\defeq\max_{\mmu'\in\calU}f(\vv,\mmu’)-\min_{\vv'\in\calV}f(\vv’,\mmu)$.}

An $\eps$-approximate solution of the minimax problem is a pair of feasible solution $(\vv^\eps,\mmu^\eps)\in\calV\times\calU$ with its duality gap bounded by $\eps$, i.e. $\Gap(\vv^\eps,\mmu^\eps)\le\eps.$ An expected $\eps$-approximate solution is one satisfying $\E\Gap(\vv^\eps,\mmu^\eps)\le\eps.$ %

\section{Stochastic Mirror Descent Framework}
\label{sec:framework}

In this section, we consider the following $\ell_\infty$-$\ell_1$ bilinear games as an abstraction of the MDP minimax problems of interest. Such games are induced by one player minimizing over the box domain ($\ell_\infty$) and the other maximizing over the simplex domain ($\ell_1$) a bilinear objective:
\begin{equation}
	\min_{\xx\in \B^n_b}\max_{\yy\in \Delta^m}f(\xx,\yy)\defeq \yy^\top\M\xx+\bb^\top\xx-\cc^\top\yy ,\label{def-minimax-framework}
\end{equation}
where throughout the paper we use $\B^n_b\defeq b\cdot[-1,1]^n$ to denote the box constraint, and $\Delta^m$ to denote the simplex constraint for $m$-dimensional space.

\notarxiv{

We study the efficiency of coordinate stochastic mirror descent algorithms onto this $\ell_\infty$-$\ell_1$ minimax problem. The analysis follows from extending a fine-grained analysis of mirror descent with Bregman divergence using local norm arguments in~\citet{Shalev12,CJST19} to the $\ell_\infty$-$\ell_1$ domain. We defer all proofs in this section to Appendix~\ref{app:sec3}.

}

\arxiv{

We study the efficiency of coordinate stochastic mirror descent (SMD) algorithms onto this $\ell_\infty$-$\ell_1$ minimax problem. The analysis follows from extending a fine-grained analysis of mirror descent with Bregman divergence using local norm arguments in~\citet{Shalev12,CJST19,CJST20} to the $\ell_\infty$-$\ell_1$ domain. (See Lemma~\ref{lem:mirror-descent-l1} and Lemma~\ref{lem:mirror-descent-l2} for details.) %
}

At a given iterate $(\xx,\yy) \in \B^n_b \times \Delta^m$, our algorithm computes an estimate of the gradients for both sides defined as %
\arxiv{
\begin{equation}
\begin{aligned}
	& g\x(\xx,\yy)\defeq\M^\top\yy+\bb\in\R^n\quad \text{(gradient for $x$ side, $g\x$ in shorthand)};\\
	&  g\y(\xx,\yy)\defeq-\M\xx+\cc\in\R^m \quad \text{(gradient for $y$ side, $g\y$ in shorthand)}.
\end{aligned}
\end{equation}
}
\notarxiv{
$g\x(\xx,\yy)\defeq\M^\top\yy+\bb\in\R^n$ ($x$-gradient, or $g\x$); $g\y(\xx,\yy)\defeq-\M\xx+\cc\in\R^m$ ($y$-gradient, or $g\y$).
}

The norm we use to measure these gradients are induced by Bregman divergence, a natural extension of Euclidean norm. For our analysis we choose to use 
\arxiv
{the following divergence terms\notarxiv{ ($\forall \xx,\xx'\in\B^n_b$, $\forall \yy,\yy'\in\Delta^m$) }:
\begin{equation}
\begin{aligned}
	\text{Euclidean distance for $x$ side: } & V_{\xx}(\xx')\defeq\frac{1}{2}\ltwo{\xx-\xx'}^2,\quad \forall \xx,\xx'\in\B^n_b;\\
	\text{KL divergence for $y$ side: } & V_{\yy}(\yy')\defeq\sum_{i\in[m]}y_i\log(y'_{i}/y_{i}),\quad \forall \yy,\yy'\in\Delta^m,
\end{aligned}
\end{equation}}
\notarxiv{ $V_{\xx}(\xx')\defeq\frac{1}{2}\ltwo{\xx-\xx'}^2$ and $V_{\yy}(\yy')\defeq\sum_{i\in[m]}y_i\log(\tfrac{y'_{i}}{y_{i}})$ (KL-divergence),}
which are also common practice~\citep{W17d,W17m,nesterov2005smooth} catering to the geometry of each domain, and induce the dual norms on the gradients in form $\norm{g\x}\defeq\ltwo{g\x}=\sqrt{\sum_{j\in[n]}{{g\x_j}^2}}$ (standard $
\ell_2$-norm)for $x$ side, and $\norm{g\y}_{\yy'}^2\defeq\sum_{i\in[m]}y'_{i}(g\y_{i})^2$ (a weighted $\ell_2$-norm) for $y$ side.

To describe the properties of estimators needed for our algorithm, we introduce the following definition of \emph{bounded estimator} as follows. 

\begin{definition}[Bounded Estimator]
Given the following properties on mean, scale and variance of an estimator:\\
(i) unbiasedness: $\E\tilde{g}=g$;\\
(ii) bounded maximum entry: $\linf{\tilde{g}}\le c$ with probability $1$;\\
(iii) bounded second-moment: $\E\norm{\tilde{g}}^2\le v$\\
we call $\tilde{g}$ a \emph{$(v,\norm{\cdot})$-bounded estimator of $g$} if satisfying $(i)$ and $(iii)$, call it and a \emph{$(c,v,\norm{\cdot}_\Delta^m)$-bounded estimator of $g$} if it satisfies $(i)$, $(ii)$, and also $(iii)$ with local norm $\norm{\cdot}_{\yy}$ for all $\yy\in\Delta^m$.  
\end{definition}

Now we give Algorithm~\ref{alg:framework}, our general algorithmic framework for solving \eqref{def-minimax-framework} given efficient bounded estimators for the gradient. Its theoretical guarantees are given in Theorem~\ref{thm:framework} which bounds the number of iterations needed to obtain expected $\eps$-approximate solution. \arxiv{We remark that the proof strategy and consideration of ghost-iterates stems from a series of work offering standard analysis for saddle-point problems~\cite{nemirovski2009robust,CJST19,CJST20}. }

\arxiv{
\begin{algorithm}
\caption{SMD for $\ell_\infty$-$\ell_1$ saddle-point problem}
	\label{alg:framework}
\begin{algorithmic}[1]
 \STATE \textbf{Input:} Desired accuracy $\eps$, primal domain size $b$
 \STATE	\textbf{Output:} An expected $\eps$-approximate solution $(\xx^\eps,\yy^\eps)$ for problem~\eqref{def-minimax-framework}.
\STATE \textbf{Parameter:} Step-size $\eta\x\le \tfrac{\eps}{4v\x}$, $\eta\y\le\tfrac{\eps}{4v\y}$,  total iteration number $T\ge \max\{\tfrac{16nb^2}{\eps\eta_x},\tfrac{8\log m}{\eps\eta_y}\}$.
		\FOR{$t=1,\ldots,T-1$}
			\STATE
			 Get $\tilde{g}\x_t$
			  as a $(v\x,\norm{\cdot}_2)$-bounded estimator of $g\x{(\xx_t,\yy_t)}$
			  \STATE Get $\tilde{g}\y_t$ as a $(\tfrac{2v\y}{\eps},v\y,\norm{\cdot}_{\Delta^m})$-bounded estimator of $g\y{(\xx_t,\yy_t)}$
			\STATE Update $\xx_{t+1} \leftarrow \argmin\limits_{\xx\in \B^n_b} \langle \eta\x \tilde{g}\x_t, \xx \rangle + V_{\xx_{t}}(\xx)$, and $\yy_{t+1} \leftarrow \argmin\limits_{\yy\in\Delta^m} \langle \eta\y \tilde{g}\y_t, \yy \rangle + V_{\yy_{t}}(\yy)$
		\ENDFOR
		\STATE \textbf{Return} $(\xx^\eps,\yy^\eps)\leftarrow\frac{1}{T}\sum_{t\in[T]} (\xx_t,\yy_t)$
\end{algorithmic}
\end{algorithm}
}

\notarxiv{
\begin{algorithm}
\caption{SMD for $\ell_\infty$-$\ell_1$ game}
	\label{alg:framework}
\begin{algorithmic}[1]
 \STATE \textbf{Input:} Desired accuracy $\eps$, primal domain size $b$, $(v\x,\norm{\cdot}_2)$-bounded estimator $g\x$, $(\tfrac{4v\y}{\eps},v\y,\norm{\cdot}_{\Delta^m})$-bounded estimator $g\y$
 \STATE	\textbf{Output:} An expected $\eps$-approximate solution $(\xx^\eps,\yy^\eps)$ for problem~\eqref{def-minimax-framework}.
\STATE \textbf{Parameter:} Step-size $\eta\x$, $\eta\y$,  total iteration number $T$.
		\FOR{$t=0,\ldots,T-1$}
			\STATE
			 Get $\tilde{g}\x$ estimator for $g\x$, $\tilde{g}\y$ estimator for $g\y$
			\STATE Update $\xx_{t+1} \leftarrow \argmin\limits_{\xx\in \B^n_b} \langle \eta\x \tilde{g}\x{(\xx_{t},\yy_{t})}, \xx \rangle + V_{\xx_{t}}(\xx)$
			\STATE Update $\yy_{t+1} \leftarrow \argmin\limits_{\yy\in\Delta^m} \langle \eta\y \tilde{g}\y{(\xx_{t},\yy_{t})}, \yy \rangle + V_{\yy_{t}}(\yy)$ %
		\ENDFOR
		\STATE \textbf{Return} $(\xx^\eps,\yy^\eps)\leftarrow\frac{1}{T}\sum_{t\in[T]} (\xx_t,\yy_t)$
\end{algorithmic}
\end{algorithm}
}

\begin{theorem}\label{thm:framework}
	Given an $\ell_\infty$-$\ell_1$ game, i.e. \eqref{def-minimax-framework}, and desired accuracy $\eps$, $(v\x,\norm{\cdot}_2)$-bounded estimators $\tilde{g}\x$ of $g\x$, and $(\frac{2v\y}{\eps},v\y,\norm{\cdot}_{\Delta^m})$-bounded estimators $\tilde{g}\y$ of $g\y$, Algorithm~\ref{alg:framework} with choice of parameters $\eta_x \le \tfrac{\eps}{4v\x} $, $\eta_y \le \tfrac{\eps}{4v\y}$ outputs an expected $\eps$-approximate optimal solution within any iteration number $T\ge \max\{\tfrac{16nb^2}{\eps\eta_x},\tfrac{8\log m}{\eps\eta_y}\}$. 
\end{theorem}

\arxiv{

We first recast a few standard results on the analysis of mirror-descent using local norm~\citep{Shalev12}, which we use for proving Theorem~\ref{thm:framework}. These are standard regret bounds for $\ell_2$ and simplex respectively. First, we provide the well-known regret guarantee for $\xx\in\ball^n$, when choosing $V_{\xx}(\xx')\defeq\frac{1}{2}\ltwo{\xx-\xx'}^2$. 

\begin{lemma}[cf. Lemma~12 in~\citet{CJST19}, restated]\label{lem:mirror-descent-l2}
	Let $T\in\N$ and let 
	$\xx_1\in\xset$, $\ggamma_1,\ldots,\ggamma_{T}\in\xset^*$, $V$ is 1-strongly convex in $\|\cdot\|_2$. 
	The sequence 
	$\xx_2,\ldots,\xx_T$ defined 
	by 
	\begin{equation*}
	\xx_{t+1} = \argmin_{\xx\in\xset}\left\{ 
	\inner{\ggamma_{t}}{\xx}+V_{\xx_{t}}(\xx)
		\right\}
	\end{equation*}
	satisfies for all $\xx\in\xset$ (overloading notations to denote $\xx_{T+1}\defeq \xx$),
	\begin{flalign*}
	\sum_{t\in[T]}\inner{\ggamma_t}{\xx_t -\xx} & \le  V_{\xx_1}(\xx) +
	\sum_{t\in[T]}
	\left\{  \inner{\ggamma_{t}}{\xx_{t} - \xx_{t+1}}  - V_{\xx_{t}}(\xx_{t+1})\right\}\\
	&  \le V_{\xx_1}(\xx) + \frac{1}{2}\sum_{t\in[T]} \ltwo{\ggamma_t}^2.
	\end{flalign*}
\end{lemma}

Next, one can show a similar property holds true for $\yy\in\Delta^m$, by choosing KL-divergence as Bregman divergence $V_{\yy}(\yy')\defeq\sum_{i\in[m]}y_i\log(y_i'/y_i)$, utilizing local norm $\norm{\cdot}_{\yy'}$.

\begin{lemma}[cf. Lemma 13 in~\citet{CJST19}, immediate consequence]\label{lem:mirror-descent-l1}
	Let $T\in\N$, 
	$\yy_1\in\yset$, $\ggamma_1,\ldots,\ggamma_T\in\yset^*$ satisfying $\linf{
	\ggamma_t}\le1.79,\forall t\in[T]$, and $V_{\yy}(\yy')\defeq\sum_{i\in[m]}y_i\log(y_i'/y_i)$. 
	The sequence 
	$\yy_2,\ldots,\yy_T$ defined 
	by
	\begin{equation*}
	\yy_{t+1} = \argmin_{\yy\in\zset}\left\{ 
	\inner{\ggamma_{t}}{\yy}+V_{\yy_{t}}(\yy)
		\right\}
	\end{equation*}
	satisfies for all $\yy\in\yset$ (overloading notations to denote $\yy_{T+1}\defeq \yy$),
	\begin{flalign*}
	\sum_{t\in[T]}\inner{\ggamma_t}{\yy_t -\yy} \le & V_{\yy_1}(\yy) +
	\sum_{t\in[T]}
	\left\{  \inner{\ggamma_{t}}{\yy_{t} - \yy_{t+1}}  - V_{\yy_{t}}(\yy_{t+1})\right\}\\
	\le & V_{\yy_1}(\yy) + \frac{1}{2}\sum_{t\in[T]} \norm{\ggamma_t}_{\yy_t}^2.
	\end{flalign*}
\end{lemma}

Leveraging these lemmas we prove Theorem~\ref{thm:framework}.

\begin{proof}[Proof of Theorem~\ref{thm:framework}]

For simplicity we use $g\x_t$, $g\y_t$, $\tilde{g}\x_t$, $\tilde{g}\y_t$ for shorthands of $g\x(\xx_t,\yy_t)$, $g\y(\xx_t,\yy_t)$, $\tilde{g}\x(\xx_t,\yy_t)$, $\tilde{g}\y(\xx_t,\yy_t)$ throughout the proof, similar as in Algorithm~\ref{alg:framework}. By the choice of $\eta\y$ and conditions, one can immediately see that 
\[
\linf{\eta\y\tilde{g}\y_t}\le1/2.
\]
Thus we can use regret bound of stochastic mirror descent with local norms in Lemma~\ref{lem:mirror-descent-l1} and Lemma~\ref{lem:mirror-descent-l2} which gives
\begin{equation}
\label{eq:reg-term1}
\begin{aligned}
\sum\limits_{t\in[T]}\langle\eta\x \tilde{g}\x_t, & \xx_t-\xx\rangle
\le V_{\xx_1}(\xx)+\frac{{\eta\x}^2}{2}\sum\limits_{t\in[T]}\|\tilde{g}\x_t\|_2^2,\\
\sum\limits_{t\in[T]}\langle\eta\y \tilde{g}\y_t, & \yy_t-\yy\rangle
\le  V_{\yy_1}(\yy) +\frac{{\eta\y}^2}{2}\sum\limits_{t\in[T]}\norm{\tilde{g}\y_t}_{\yy_t}^2.
\end{aligned}	
\end{equation}

Now, let $\hat{g}\x_t \defeq g\x_t - \tilde{g}\x_t$ and $\hat{g}\y_t \defeq g\y_t - \tilde{g}\y_t$, defining the sequence $\hat{\xx}_1, \hat{\xx}_2, \ldots, \hat{\xx}_T$ and $\hat{\yy}_1, \hat{\yy}_2, \ldots, \hat{\yy}_T$ according to
\begin{equation*}
\begin{aligned}
\hat{\xx}_1 & = \xx_1, & \hat{\xx}_{t+1} & = \arg\min\limits_{\xx\in  \B_b^n} \langle \eta\x \hat{g}\x_{t}, \xx \rangle + V_{\hat{\xx}_{t}}(\xx);\\
\hat{\yy}_1 & = \yy_1, & \hat{\yy}_{t+1} & = \arg\min\limits_{\yy\in\Delta^m} \langle \eta\y \hat{g}\y_{t}, \yy \rangle + V_{\hat{\yy}_{t}}(\yy).\\
\end{aligned}
\end{equation*}
Using a similar argument for $\hat{g}_t\y$ satisfying
\[\|\eta\y\hat{g}_t\y\|_\infty\le \|\eta\y\tilde{g}_t\x\|_\infty + \|\eta\y g_t\y\|_\infty  = \|\eta\y\tilde{g}_t\x\|_\infty + \|\eta\y \E \tilde{g}_t\y\|_\infty \le 2 \|\eta\y\tilde{g}_t\x\|_\infty \leq 1,\]
we obtain 
\begin{equation}
\label{eq:reg-term2-framework}
\begin{aligned}
\sum\limits_{t\in[T]}\langle\eta\x \hat{g}\x_t, \hat{\xx}_t-\xx\rangle
& \le V_{\xx_0}(\xx) +\frac{{\eta\x}^2}{2}\sum\limits_{t\in[T]}\|\hat{g}\x_t\|_2^2,\\
\sum\limits_{t\in[T]}\langle\eta\y \hat{g}\y_t, \hat{\yy}_t-\yy\rangle & 
\le  V_{\yy_0}(\yy) +\frac{{\eta\y}^2}{2}\sum\limits_{t\in[T]}\norm{\hat{g}\y_t}_{\hat{\yy}_t}^2.
\end{aligned}	
\end{equation}

Since  $g_t\x=\tilde{g}_t\x+\hat{g}_t\x$ and $g_t\y=\tilde{g}_t\y+\hat{g}_t\y$, rearranging yields
\begin{equation}
\label{eq:reg-total-framework}
\begin{aligned}
& \sum\limits_{t\in[T]}\left[\langle g\x_t,\xx_t-\xx \rangle+ \langle g\y_t, \yy_t-\yy\rangle\right]\\
= &\sum\limits_{t\in[T]}\left[\langle\tilde{g}\x_t,\xx_t-\xx \rangle+ \langle\tilde{g}\y_t, \yy_t-\yy\rangle\right] + \sum\limits_{t\in[T]}\left[\langle\hat{g}\x_t,\xx_t-\xx \rangle+ \langle\hat{g}\y_t, \yy_t-\yy\rangle\right]\\
\le & \ \ \frac{2}{\eta\x} V_{\xx_1}(\xx) + \sum_{t\in[T]}\left[\frac{\eta\x}{2} \|\tilde{g}\x_t\|_2^2 + \frac{\eta\x}{2} \|\hat{g}\x_t\|_2^2\right]  + \sum_{t\in[T]}\langle\hat{g}\x_t,\xx_t-\hat{\xx}_t\rangle \\
& + \frac{2}{\eta\y} V_{\yy_1}(\yy) + \sum_{t\in[T]}\left[\frac{\eta\y}{2}\norm{\tilde{g}\y_t}_{\yy_t}^2 +\frac{\eta\y}{2}\norm{\hat{g}\y_t}_{\hat{\yy}_t}^2\right] + \sum_{t\in[T]}\langle\hat{g}\y_t,\yy_t-\hat{\yy}_t\rangle.
\end{aligned}	
\end{equation}
where we use the regret bounds in Eq.~\eqref{eq:reg-term1},~\eqref{eq:reg-term2-framework} for the inequality. 

Now take supremum over $(\xx,\yy)$ and then take expectation on both sides, we get
\begin{align*}
	& \frac{1}{T}\E\sup_{\xx,\yy}\left[\sum\limits_{t\in[T]}\langle g\x_t,\xx_t-\xx \rangle+ \sum\limits_{t\in[T]}\langle g\y_t, \yy_t-\yy\rangle\right]\\
	\stackrel{(i)}{\le} & \sup_{\xx}\frac{2}{\eta\x T}V_{\xx_0}(\xx)+\eta\x v\x+\sup_{\yy}\frac{2}{\eta\y T}V_{\yy_0}(\yy)+\eta\y v\y\\
	\stackrel{(ii)}{\le} & \frac{4nb^2}{\eta\x T}+\eta\x v\x+\frac{2\log m}{\eta\y T} + \eta\y v\y\\
	 \stackrel{(iii)}{\le} & \eps,
\end{align*}
where we use $(i)$ $\E[\langle\hat{g}\x_t,{\xx}_t-\hat{\xx}_t \rangle|1,2,\cdots,T]=0$, $\E[\langle\hat{g}\y_t,{\yy}_t-\hat{\yy}_t\rangle |1,2,\cdots,T]=0$ by conditional expectation, that $\E\|\hat{g}\x_t\|_2^2\le \E\|\tilde{g}\x_t\|_2^2$, $\E[\sum_i[\hat{\yy}_t]_i[\hat{g}\y_t]_i^2]\le \E[\sum_i[\hat{\yy}_t]_i[\tilde{g}\y_t]_i^2]$ due to the fact that $\E[(X-\E X)^2]\le \E[X^2]$ elementwise and properties of estimators as stated in condition; $(ii)$ $V_{\xx_0}(\xx)\defeq\tfrac{1}{2}\ltwo{\xx-\xx_0}^2\le 2nb^2$, $V_{\yy_0}(\yy)\le\log m$ by properties of KL-divergence; $(iii)$ the choice of $\eta\x=\frac{\eps}{4 v\x}$, $\eta\y=\frac{\eps}{4 v\y}$, and $T\ge\max(\tfrac{16nb^2}{\eps \eta\x},\tfrac{8\log m}{\eps \eta\y})$.

Together with the bilinear structure of problem and choice of $\xx^\eps=\tfrac{1}{T}\sum_{t\in[T]}\xx_t$, $\yy^\eps=\tfrac{1}{T}\sum_{t\in[T]}\yy_t$ we get
$\E[\Gap(\xx^\eps,\yy^\eps)] \le \eps$, 
proving the output $(\xx^\eps,\yy^\eps)$ is indeed an expected  $\eps$-approximate solution to the minimax problem~\eqref{def-minimax-framework}.
\end{proof}

}

\notarxiv{

\begin{proof-sketch}
For simplicity we only include a proof sketch here and defer the complete proof details to Appendix~\ref{app:framework}.\\
\emph{Regret bounds with local norms.} The core statement is a standard regret bound using local norms (see Lemma~\ref{lem:mirror-descent-l2} and Lemma~\ref{lem:mirror-descent-l1}) which summing together gives the following guarantee \notarxiv{(let $\tilde{g}\x_t,\tilde{g}\y_t$ denote $\tilde{g}\x(\xx_t,\yy_t),\tilde{g}\y(\xx_t,\yy_t)$)}
\notarxiv{
\begin{equation}\label{eq:regret-sketch}
	\begin{aligned}
		& \sum_{t\in[T]}\langle\tilde{g}\x_t, \xx_t-\xx\rangle +\sum_{t\in[T]}\langle \tilde{g}\y_t, \yy_t-\yy\rangle\\
		\le &  \tfrac{V_{\xx_0}(\xx)}{\eta\x} +\tfrac{\sum\limits_{t=0}^T{\eta\x}\|\tilde{g}\x_t\|_2^2}{2}+\tfrac{V_{\yy_0}(\yy)}{\eta\y} +\tfrac{\sum\limits_{t=0}^T\eta\y\norm{\tilde{g}\y_t}_{\yy_t}^2}{2}.
	\end{aligned}
\end{equation}
}
Note one needs the bounded maximum entry condition for $\tilde{g}\y$ as the condition to use Lemma~\ref{lem:mirror-descent-l1}.\\
\emph{Domain size.} The domain size can be bounded as $\max_{\xx\in \B_b^n}V_{\xx_0}(\xx)\le 2nb^2 $, $\max_{\yy\in\Delta^m}V_{\yy_0}(\yy)\le \log m$ by definition of their corresponding Bregman divergences.\\
\emph{Second-moment bounds.}
This is given through the bounded second-moment properties of estimators directly.\\
\emph{Ghost-iterate analysis.\footnote{For standard SMD on convex problems this step is unnecessary. One can directly use conditional expectation by fixing $\xx=\xx^*$. However, for saddle-point problems, the same technique only gives $\max_{\xx,\yy}\E[\text{regret}]\le\eps$.  The ghost iterates analysis is standard~\citep{N04,CJST19} and necessary to get a bound in terms of $\E\max_{\xx,\yy}[\text{regret}]\le\eps$ instead.} }
In order to substitute $\tilde{g}\x,\tilde{g}\y$ with $g\x,g\y$ for LHS of Eq.~\eqref{eq:regret-sketch}, one can apply the regret bounds again to ghost iterates generated by taking gradient step with $\hat{g}=g-\tilde{g}$ coupled with each iteration. The additional terms coming from this extra regret bounds are in expectation $0$ through conditional expectation computation.\\
\emph{Optimal tradeoff.} One pick $\eta_x,\eta_y,T$ accordingly to get the desired guarantee as stated in Theorem~\ref{thm:framework}.
\end{proof-sketch}

}

Now we design gradient estimators assuming certain sampling oracles to ensure good bounded properties. 
\arxiv{More concretely, we offer one way to construct the gradient estimators and prove its properties and the implied algorithmic complexity.

For $\xx$-side, we consider
\begin{equation}
\label{eq:est-g-x}
\begin{aligned}
	& \text{Sample }i,j \text{ with probability }p_{i j}\defeq y_i\cdot\frac{|M_{ij}|}{\sum_{j}|M_{ij}|},\\
	& \text{sample $j'$} \text{ with probability }p_{j'}\defeq\frac{|b_{j'}|}{\lones{\bb}},\\
  & \text{set } \tilde{g}\x(\xx,\yy) = \frac{M_{i j}y_{i}}{p_{i j}}\ee_{j}+\frac{b_{j'}}{p_{j'}}\ee_{j'},
\end{aligned}	
\end{equation}
which has properties as stated in Lemma~\ref{lem:est-property-x}.

\begin{lemma}\label{lem:est-property-x}
Gradient estimator $\tilde{g}\x$ specified in~\eqref{eq:est-g-x} is a $(v\x,\norm{\cdot}_2)$-bounded estimator, with \[v\x=2\left[\lones{\bb}^2+\linf{\M}^2\right].\]
\end{lemma}

\begin{proof}
The unbiasedness follows directly by definition. For bound on second-moment, one sees
\begin{align*}
\E\ltwo{\tilde{g}\x(\xx,\yy)}^2 & \stackrel{(i)}{\le}2\left[\sum_{j'}\frac{b_{j'}^2}{p_{j'}}+\sum_{i,j}\frac{M_{ij}^2y_i^2}{p_{ij}}\right] \stackrel{(ii)}{=}2\left[\lones{\bb}^2+\left(\sum_{i,j}y_i|M_{ij}|\left(\sum_j |M_{ij}|\right)\right)\right] \\
& \stackrel{(iii)}{\le}2\left[\lones{\bb}^2+\linf{\M}^2\right],
\end{align*}
where we use $(i)$ the fact that $\norm{x+y}^2\le2\norm{x}^2+2\norm{y}^2$ and taking expectation, $(ii)$ plugging in the explicit sampling probabilities as stated in~\eqref{eq:est-g-x}, and $(iii)$  Cauchy-Schwarz inequality and the fact that $\yy\in\Delta^m$. 
\end{proof}

For $\yy$-side, we consider
\begin{equation}
\label{eq:est-g-y}
\begin{aligned}
	& \text{Sample }i,j \text{ with probability }q_{i j}\defeq\frac{|M_{ij}|}{\sum_{i,j}|M_{ij}|},\\
	& \text{sample $i'$} \text{ with probability }q_{i'}\defeq \frac{|c_{i'}|}{\lones{\cc}},\\
  & \text{set } \tilde{g}\y(\xx,\yy) = -\frac{M_{i j}x_{j}}{q_{i j}}\ee_{i}+\frac{c_{i'}}{q_{i'}}\ee_{i'}.
\end{aligned}	
\end{equation}
Here we remark that we adopt the same indexing notation $i,j$ but it is independently sampled from given distributions as with ones for $\tilde{g}\x$. Such an estimator has properties stated in Lemma~\ref{lem:est-property-y}.

\begin{lemma}\label{lem:est-property-y}
Gradient estimator $\tilde{g}\y$ specified in~\eqref{eq:est-g-y} is a $(c\y,v\y,\norm{\cdot}_{\Delta^m})$-bounded estimator, with \[c\y=m(b\linf{\M}+\linf{\cc}),\ v\y=2m\left[\linf{\cc}^2+b^2\linf{\M}^2\right].\]
\end{lemma}

\begin{proof}
The unbiasedness follows directly by definition. For bounded maximum entry, one has
\[
\linf{\tilde{g}\y}\le\sum_{i,j}|M_{ij}x_j|+\lones{\cc}\le m(b\linf{\M}+\linf{\cc}),
\]
by definition of the probability distributions and $x_j\in\B_b^n,\cc\in\R^m$.

For bound on second-moment in local norm with respect to arbitrary $\yy'\in\Delta^m$, one has
\begin{align*}
\E\norm{\tilde{g}\y(\xx,\yy)}_{\yy'}^2 \stackrel{(i)}{\le} & 2\left[\sum_{j'}y'_{i'}\frac{c_{i'}^2}{q_{i'}}+\sum_{i,j}y'_{i}\frac{M_{ij}^2x_j^2}{q_{ij}}\right]\\
\stackrel{(ii)}{=} & 2\left[\left(\sum_{i'}y'_{i'}c_{i'}\right)\lones{\cc}+\left(\sum_{i,j}y'_i|M_{ij}|x_j^2\right)\left(\sum_{i,j}|M_{ij}|\right)\right]\\
\stackrel{(iii)}{\le} & 2\left[m\linf{\cc}^2+mb^2\linf{\M}^2\right],
\end{align*}
where we use $(i)$ the fact that $\norm{x+y}^2\le2\norm{x}^2+2\norm{y}^2$ and taking expectation, $(ii)$ plugging in the explicit sampling probabilities as stated in~\eqref{eq:est-g-y}, and $(iii)$  Cauchy-Schwarz inequality and the fact that $\yy'\in\Delta^m$, $\cc\in\R^m$, $\xx\in\B^n_b$. 
\end{proof}
}

When $\xx\in\B_1^n$, this leads to the theoretical guarantee as stated formally in Corollary~\ref{cor:linf-reg}. \notarxiv{We defer design of estimators and all proofs to Appendix~\ref{app:framework-est} and simply state the theoretical runtime guarantee of Algorithm~\ref{alg:framework} here.}

\begin{corollary}\label{cor:linf-reg}
	Given an $\ell_\infty$-$\ell_1$ game \eqref{def-minimax-framework} with domains $\xx\in\B^n_1$, $\yy\in\Delta^m$, $\eps\in(0,1)$ and $\linf{\M}+\linf{\cc}=\Omega(1)$. If one has all sampling oracles needed with sampling time $O(\mathrm{T}_{\mathrm{samp}})$~\footnote{Note all the sampling oracles needed are essentially $\ell_1$ samplers proportional to the matrix / vector entries, and an $\ell_1$ sampler induced by $y\in\Delta^m$. These samplers with $\widetilde{O}(1)$ cost per sample can be built with additional preprocessing in $\widetilde{O}(\nnz(\M) + n + m)$ time.
	}, Algorithm~\ref{alg:framework} with certain gradient estimators (see~\eqref{eq:est-g-x} and~\eqref{eq:est-g-y}) finds an expected $\eps$-approximate solution with a number of samples bounded by  
	 \arxiv{\[O([(n+m\log m)\linf{\M}^2+n\lones{\bb}^2+m\log m\linf{\cc}^2]\cdot\eps^{-2}\cdot\mathrm{T}_{\mathrm{samp}}).\]}
	 \notarxiv{$O([(n+m\log m)\linf{\M}^2+n\lones{\bb}^2+m\log m\linf{\cc}^2]\cdot\eps^{-2}).$}
	 Further the runtime is proportional to the number of samples times the cost per sample.
\end{corollary}

\arxiv{
\begin{proof}[Proof of Corollary~\ref{cor:linf-reg}]
In light of Theorem~\ref{thm:framework} with Lemma~\ref{lem:est-property-x} and Lemma~\ref{lem:est-property-y}, whenever $\eps\in(0,1)$, $b\linf{\M}+\linf{\cc}=\Omega(1)$, gradient estimators in~\eqref{eq:est-g-x} and~\eqref{eq:est-g-y} satisfy the desired conditions. As a result, one can pick 
\begin{align*}
\eta\x & = \Theta\left(\frac{\eps}{\lones{\bb}^2+\linf{\M}^2}\right),\eta\y = \Theta\left(\frac{\eps}{m\left(\linf{\cc}^2+b^2\linf{\M}^2\right)}\right),\\
T & =\O{\frac{(n+m\log m)b^2\linf{\M}^2+nb^2\lones{\bb}^2+m\log m\linf{\cc}^2}{\eps^2}},
\end{align*}
to get an expected $\eps$-approximate solution to the general $\ell_
\infty$-$\ell_1$ bilinear saddle-point problem~\eqref{def-minimax-framework}, proving the corollary.
\end{proof}
}

Finally, we remark that one can also use Algorithm~\ref{alg:framework} to solve $\ell_\infty$-regression, i.e. the problem of finding
\arxiv{\[\xx^*\defeq\arg\min_{\xx\in\B^n_1}\linf{\M\xx-\cc}\]}
\notarxiv{$\xx^*\defeq\arg\min_{\xx\in\B^n_1}\linf{\M\xx-\cc}$}
	 by simply writing it in equivalent~ minimax form of 
	 \arxiv{
\[\min_{\xx\in\B_1^n}\max_{\yy\in\Delta^m}\yy^\top(\hat{\M}\xx-\hat{\cc)},\hat{\M}\defeq[\M;-\M],\hat{\cc}\defeq[\cc;-\cc].\]}
\notarxiv{ $\min_{\xx\in\B_1^n}\max_{\yy\in\Delta^m}\yy^\top(\hat{\M}\xx-\hat{\cc)}$ where $\hat{\M}\defeq[\M;-\M]$ and $\hat{\cc}\defeq[\cc;-\cc]$.}
\begin{remark}\label{rem:linf-reg}
Algorithm~\ref{alg:framework} produces an expected $\eps$-approximate solution $\xx^\eps$ satisfying
\arxiv{\[\E\linf{\M\xx^\eps-\cc}\le \linf{\M\xx^*-\cc}+\eps,\]}
\notarxiv{$\E\linf{\M\xx^\eps-\cc}\le \linf{\M\xx^*-\cc}+\eps,$}
within runtime 
\arxiv{
\begin{align*}
	 & \Otilb{\left[(m+n)\linf{\M}^2+m\linf{\cc}^2\right]\cdot\eps^{-2}\cdot \mathrm{T}_{\mathrm{samp}}}.
\end{align*}}
\notarxiv{$\Otilb{\left[(m+n)\linf{\M}^2+m\linf{\cc}^2\right]\cdot\eps^{-2}}$.}
\end{remark}

\section{Mixing AMDPs}
\label{sec:mixingAMDP}

In this section we show how to utilize framework in Section~\ref{sec:framework} for mixing AMDPs to show efficient primal-dual algorithms that give an approximately optimal policy. In Section~\ref{ssec:bound-mixing} we specify the choice of $M$ in minimax problem~\eqref{def-minimax-mixing} by bounding the operator norm to give a domain that $\vv^*$ lies in. In Section~\ref{ssec:est-mixing} we give estimators for both sides for solving~\eqref{def-minimax-mixing}, which is similar to the estimators developed in Section~\ref{sec:framework}. In Section~\ref{ssec:sub-mixing} we show how to round an $\eps$-optimal solution of~\eqref{def-minimax-mixing} to an $\Theta(\eps)$-optimal policy. Due to the similarity of the approach and analysis, we include our method for solving DMDPs and its theoretical guarantees in Appendix~\ref{sec:DMDP}.

\subsection{Bound on Matrix Norm}
\label{ssec:bound-mixing}

We first introduce Lemma~\ref{lem:norm-bounds-mixing} showing that the mixing assumption~\ref{assum} naturally leads to $\ell_\infty$-norm bound on the interested matrix, which is useful in both in deciding $M$ and in proving Lemma~\ref{lem:approx-mixing} in Section~\ref{ssec:sub-mixing}.

\begin{restatable}{lemma}{restateNormMixing}
\label{lem:norm-bounds-mixing}
Given a mixing AMDP, policy $\pi$, and its probability transition matrix $\PP^\pi\in\R^{\calS\times \calS}$ and stationary distribution $\nnu^\pi$,
\arxiv{
\[\linf{(\II-\PP^{\pi}+\1(\nnu^\pi)^\top)^{-1}}\le 2\tmix.\]}
\notarxiv{$\linf{(\II-\PP^{\pi}+\1(\nnu^\pi)^\top)^{-1}}\le 2\tmix.$}
\end{restatable}

In order to prove Lemma~\ref{lem:norm-bounds-mixing}, we will first give a helper lemma adapted from  \citet{cohen2016faster} capturing the property of $\II-\PP^\pi+\nnu^\pi\1^\top$. Compared with the lemma stated there, we are removing an additional assumption about strong connectivity of the graph as it is not necessary for the proof.

\begin{lemma}[cf. Lemma 23 in \citet{cohen2016faster}, generalized]\label{lem:RW-bounded-norm}
	For a probabilistic transition matrix $\PP^\pi$ with mixing time $\tmix$ as defined in Assumption~\ref{assum} and stationary distribution $\nnu^\pi$, one has for all non-negative integer $k\ge \tmix$,
	\[\linf{(\PP^\pi)^{k}-\1(\nnu^\pi)^\top} \le \left(\frac{1}{2}\right)^{\lfloor k/\tmix\rfloor}.\]
\end{lemma}

We use this lemma and additional algebraic properties involving operator norms and mixing time for the proof of Lemma~\ref{lem:norm-bounds-mixing}, formally as follows.

\begin{proof}[Proof of Lemma~\ref{lem:norm-bounds-mixing}] 
Denote $\hat{\PP}\defeq \PP^\pi-\1(\nnu^\pi)^\top$, we first show the following equality.
\begin{equation}\label{eq:inverse-equality}
(\II-\hat{\PP})^{-1}\stackrel{(i)}{=}\sum_{k=0}^\infty\sum_{t=k\tmix+1}^{(k+1)\tmix}\hat{\PP}^t=\sum_{k=0}^\infty\sum_{t=k\tmix+1}^{(k+1)\tmix}\left((\PP^\pi)^t-\1(\nnu^\pi)^\top\right), 
\end{equation}
To show the equality $(i)$, observing that by Lemma~\ref{lem:RW-bounded-norm}
\[
\linf{(\PP^\pi)^{k}-\1(\nnu^\pi)^\top}\le \left(\frac{1}{2}\right)^{\lfloor k/\tmix\rfloor},
\]
and thus by triangle inequality of matrix norm
\begin{align*}
\left\|
\sum_{k=0}^\infty\sum_{t=k\tmix+1}^{(k+1)\tmix}\left((\PP^\pi)^t-\1(\nnu^\pi)^\top\right)
\right\|_\infty 
&\leq 
\sum_{k=0}^\infty\sum_{t=k\tmix+1}^{(k+1)\tmix}
\norm{\left((\PP^\pi)^t-\1(\nnu^\pi)^\top\right)}_\infty
\\
&\leq 
\sum_{k=0}^\infty\sum_{t=k\tmix+1}^{(k+1)\tmix}
\left(\frac{1}{2}\right)^{k}
= 
\sum_{k = 0}^{\infty} \tmix 
\left(\frac{1}{2}\right)^{ k }
= 2 \tmix
\end{align*}
and therefore the RHS of Eq.~\eqref{eq:inverse-equality} exists.

Also one can check that 
\[
(\II-\hat{\PP})\left(\sum_{t=0}^\infty \hat{\PP}^t\right) = (\II-\hat{\PP})\left(\sum_{t=0}^\infty \hat{\PP}^t\right) = \II,\]
which indicates that equality~\eqref{eq:inverse-equality} is valid.

The conclusion thus follows directly from the matrix norm bound.
\end{proof}

This immediately implies the following corollary.

\begin{corollary}[Bound on $\vv^*$]\label{cor:bound-mixing}
For mixing AMDP~\eqref{def-Bellman-mixing}, for some optimal policy $\pi^*$ with corresponding stationary distribution $\nnu^{*}$, there exists an optimal value vector $\vv^*\perp\nnu^*$ such that 
\[
\linf{\vv^*}\le 2\tmix.
\]	
\end{corollary}
\begin{proof}
By optimality conditions $(\II-\PP^*)\vv^*=\rr^*-\bar{v}^*\1$, and $\langle\nnu^*,\vv^*\rangle=0$ one has
\[
(\II-\PP^\pi+\1(\nnu^{*})^\top)\vv^*=\rr^*-\bar{v}^*\1\]
which gives
\[\norm{\vv^*}_\infty = \norm{(\II-\PP^*+\1(\nnu^{*})^\top)^{-1}(\rr^*-\bar{v})}_\infty\le \norm{(\II-\PP^*+\1(\nnu^{*})^\top)^{-1}}_{\infty}\norm{\rr^*-\bar{v}}_\infty\le 2\tmix
\]
where the last inequality follows from Lemma~\ref{lem:norm-bounds-mixing}. 	
\end{proof}

Thus, we can safely consider the minimax problem~\eqref{def-minimax-mixing} with the additional constraint $\vv^\in\calV$, where we set $M = 2\tmix$. The extra coefficient $2$ comes in to ensure stricter primal-dual optimality conditions, which we use in Lemma~\ref{lem:approx-mixing} for the rounding.

\subsection{Design of Estimators}
\label{ssec:est-mixing}

Given domain setups, now we describe formally the gradient estimators used in Algorithm~\ref{alg:sublinear-mixing} and their properties.

For the $\vv$-side, we consider the following gradient estimator
\arxiv{
\begin{equation}
\label{eq:est-v-mixing}
\begin{aligned}
\text{Sample } & (i,a_i)\sim [\mmu]_{i,a_i}, j\sim p_{ij}(a_i).\\
\text{Set } & \tilde{g}\vsf(\vv,\mmu) = \ee_j-\ee_i.
\end{aligned}
\end{equation}
}

This is a bounded gradient estimator for the box domain.
\arxiv{\begin{lemma}
\label{lem:est-property-v-mixing}
$\tilde{g}\vsf$ defined in~\eqref{eq:est-v-mixing} is a $(2,\norm{\cdot}_2)$-bounded  estimator.
\end{lemma}}
\notarxiv{
\begin{lemma}
\label{lem:est-property-v-mixing}
$\tilde{g}\vsf$ as in~\eqref{eq:est-v-mixing} is a $(2,\norm{\cdot}_2)$-bounded  estimator.
\end{lemma}
}

\arxiv{
\begin{proof}%
For unbiasedness, direct computation reveals that 
\[
\E\left[\tilde{g}\vsf{(\vv,\mmu)}\right]  = \sum_{i,a_i,j}\mu_{i,a_i}p_{ij}(a_i)(\ee_j-\ee_i)
 =\mmu^\top(\PP-\hat{\II}).
\]
For a bound on the second-moment, note $\ltwos{\tilde{g}\vsf{(\vv,\mmu)}}^2\le2$ with probability 1 by definition, the result follows immediately.
\end{proof}
}

For the $\mmu$-side, we consider the following gradient estimator
\arxiv{
\begin{equation}
\label{eq:est-mu-mixing}
\begin{aligned}
\text{Sample } & (i,a_i)\sim\frac{1}{\A}, j\sim p_{ij}(a_i).\\
\text{Set } & \tilde{g}{\musf}{(\vv,\mmu)} =  \A(v_i-v_j-r_{i,a_i})\ee_{i,a_i}.
\end{aligned}
\end{equation}}
\notarxiv{
\begin{equation}
\label{eq:est-mu-mixing}
\begin{aligned}
\text{Sample }& (i,a_i)\sim 1/\A, j\sim p_{ij}(a_i). \\
\text{ Set } & \tilde{g}{\musf}{(\vv,\mmu)} =  \A(v_i-\gamma v_j-r_{i,a_i})\ee_{i,a_i}.
\end{aligned}
\end{equation}}
This is a bounded gradient estimator for the simplex domain.

\begin{lemma}
\label{lem:est-property-mu-mixing}
$\tilde{g}\musf$ defined in ~\eqref{eq:est-mu-mixing} is a $((2M+1)\A,9(M^2+1)\A,\norm{\cdot}_{\Delta^\calA})$-bounded  estimator.
\end{lemma}

\arxiv{
\begin{proof}
For unbiasedness, direct computation reveals that 
\begin{align*}
\E\left[\tilde{g}\musf{(\vv,\mmu)}\right]  = \sum_{i,a_i,j}p_{ij}(a_i)(v_i-v_j-r_{i,a_i})\ee_{i,a_i} = (\hat{\II}-\PP)\vv-\rr ~.
\end{align*}
For the bound on $\ell_\infty$ norm, note that with probability 1 we have $\linf{\tilde{g}\musf(\vv,\mmu)}\le (2M+1)\A$ given $|v_i-v_j-r_{i,a_i}|\le\max\{2M,2M+1\}\le 2M+1$ by domain bounds on $\vv$.
For the bound on second-moment, given any $\mmu'\in\calU$ we have
\begin{align*}
\E[\norm{\tilde{g}\musf{(\vv,\mmu)}}_{\mmu'}^2]\le\sum_{i,a_i}\frac{1}{\A}\mu_{i,a_i}'\max\left\{(2M)^2,(2M+1)^2\right\}\A^2 \le 9 (M^2+1)\A,
\end{align*}
where the first inequality follows similarly from $|v_i-v_j-r_{i,a_i}|\le \max\{2M,2M+1\},\forall i,j,a_i$.
\end{proof}
}

Theorem~\ref{thm:framework} together with guarantees of designed gradient estimators in Lemma~\ref{lem:est-property-v-mixing},~\ref{lem:est-property-mu-mixing} and choice of $M=2\tmix$ gives Corollary~\ref{cor:mixing-regret}.

\begin{corollary}
\label{cor:mixing-regret}
Given mixing AMDP tuple $\calM=(\calS,\calA,\calP,\RR)$ with desired accuracy $\epsilon \in (0,1)$, Algorithm~\ref{alg:sublinear-mixing} with parameter choice $\eta\vsf=O(\eps)$, $\eta\musf=O(\eps\tmix^{-2}\A^{-1})$ outputs an expected $\eps$-approximate solution to mixing minimax problem~\eqref{def-minimax-mixing} 
with sample complexity \arxiv{\[O({\tmix^2\A}{\eps^{-2}}\log(\A)).\]}
\notarxiv{$O({\tmix^2\A}{\eps^{-2}}\log(\A)).$}
\end{corollary}

The proof follows immediately by noticing each iteration costs $O(1)$ sample generation, thus directly transferring the total iteration number to sample complexity.

\subsection{Rounding to Optimal Policy}
\label{ssec:sub-mixing}

Now we proceed to show how to convert an $\eps$-optimal solution of~\eqref{def-minimax-mixing} to an $\Theta(\eps)$-optimal policy for~\eqref{def-lp-mixing-matrix}. First we introduce a lemma that relates the dual variable $\mmu^\eps$ with optimal cost-to-go values $\vv^*$ and expected reward $\bar{v}^*$.

\begin{lemma}
\label{lem:gap-mixing}
If $(\vv^\eps,\mmu^\eps)$ is an expected $\eps$-approximate optimal solution to mixing AMDP minimax problem~\eqref{def-minimax-mixing}, 
then for any optimal $\vv^*$ %
and $\bar{v}^*$, 
\arxiv{
\begin{align*}
 \E \left[{\mmu^\eps}^\top\left[(\hat{\II}-\PP)\vv^*-\rr\right]+\bar{v}^*\right]\le\eps.
\end{align*}}	
\end{lemma}

\arxiv{

\begin{proof}
Note by definition %
\begin{align*}
\eps\ge	\E\Gap (\vv^\eps,\mmu^\eps) \defeq \E\max_{\hat{\vv}\in\ball^\calS_{2M},\hat{\mmu}\in\Delta^\A}\biggl[ (\hat{\mmu}-\mmu^\eps)^\top((\PP-\II)\vv^\eps+\rr) + {\mmu^\eps}^\top(\PP-\II)(\vv^\eps-\hat{\vv})\biggr].
\end{align*}
When picking $\hat{\vv}=\vv^*$ and $\hat{\mmu}=\mmu^*$, i.e. optimizers of the minimax problem, this inequality yields 
\begin{align*}
	\eps & \ge \E \biggl[({\mmu}^*-\mmu^\eps)^\top((\PP-\hat{\II})\vv^\eps+\rr)  +{\mmu^\eps}^\top(\PP-\hat{\II})(\vv^\eps-\vv^*)\biggr]\\
	 & = \E\biggl[{{\mmu}^*}^\top((\PP-\hat{\II})\vv^\eps+\rr)-{\mmu^\eps}^\top\rr - {\mmu^\eps}^\top(\PP-\hat{\II})\vv^*\biggr]\\
	 & \stackrel{(i)}{=} \E\left[{\mmu^\eps}^\top\left((\hat{\II}-\PP)\vv^*-\rr\right)\right]+{\mmu^*}^\top\rr\\
	 & \stackrel{(ii)}{=} \E\left[{\mmu^\eps}^\top\left((\hat{\II}-\PP)\vv^*\right]-\rr\right)+\bar{v}^*,
\end{align*}
where we use $(i)$ the fact that ${\mmu^*}^\top(\PP-\hat{\II})=0$ by duality feasibility and $(ii)$ $\bar{v}^*\defeq{\mmu^*}^\top\rr$ by strong duality of (P) and (D) in~\eqref{def-lp-mixing-matrix}. %
\end{proof}

}

Next we transfer an optimal solution to an optimal policy, formally through Lemma~\ref{lem:approx-mixing}.

\begin{lemma}
\label{lem:approx-mixing}
	Given an $\eps$-approximate solution $(\vv^\eps,\mmu^\eps)$ for mixing minimax problem as defined in~\eqref{def-minimax-mixing}, let $\pi^\eps$ be the unique decomposition (in terms of $\llambda^\eps$) such that  $\mu^\eps_{i,a_i}=\lambda^\eps_i\cdot\pi^\eps_{i,a_i},\forall i\in \calS,a_i\in \calA_i$, where $\llambda\in\Delta^\calS,\pi^\eps_{i}\in\Delta^{\calA_i},\forall i\in\calS$. Taking $\pi\defeq\pi^\eps$ as our policy, it holds that 
	\notarxiv{$\bar{v}^*\le\E\bar{v}^\pi+3\eps.$}
	\arxiv{\[\bar{v}^*\le\E\bar{v}^\pi+3\eps.\]}
\end{lemma}

Using this fact one can prove Lemma~\ref{lem:approx-mixing} by showing the linear constraints in dual formulation (D) of \eqref{def-lp-mixing-matrix} are approximately  satisfied given an $\eps$-approximate optimal solution $(\vv^\eps,\mmu^\eps)$ to minimax problem~\eqref{def-minimax-mixing}.

\arxiv{

\begin{proof}[Proof of Lemma~\ref{lem:approx-mixing}]
Say $(\vv^\eps,\mmu^\eps)$ is an $\eps$-optimal solution in the form $\mu_{i,a_i}^\eps=\lambda_i^\eps\pi_{i,a_i}^\eps$, for some $\llambda^\eps,\pi^\eps$, we still denote the induced policy as $\pi$ and correspondingly probability transition matrix $\PP^{\pi}$ and expected reward vector $\rr^\pi$ for simplicity.  

Notice $\vv\in\calV$ by Corollary~\ref{cor:bound-mixing} and definition of $M=2\tmix$, we get $\E\lones{{\llambda^\eps}^\top(\PP^\pi-\II)}\le \frac{1}{M}\eps$ following from
\begin{align*}
	2M \cdot \E\lone{{\llambda^\eps}^\top(\PP^\pi-\II)}
	&=
	\E\ \left[
	\max_{v \in \calV} {\llambda^\eps}^\top(\PP^\pi-\II) (- \vv)
	\right] \\
	&=
	\E\ \left[
	\max_{v \in \calV} {\llambda^\eps}^\top(\PP^\pi-\II) (\vv^* - \vv)
	- {\llambda^\eps}^\top(\PP^\pi-\II) \vv^*
	\right] \\
	&\leq 
	\epsilon + \E\lone{{\llambda^\eps}^\top(\PP^\pi-\II)} \norm{v^*}_\infty
	\leq \epsilon + M \cdot \E\lone{{\llambda^\eps}^\top(\PP^\pi-\II)}.
\end{align*}
This is the part of analysis where expanding the domain size of $\vv$ from $M$ to $2M$ will be helpful.

Now  suppose that $\nnu^\pi$ is the stationary distribution under policy $\pi\defeq \pi^\eps$. By definition, this implies 
\[
{\nnu^\pi}^\top(\PP^\pi-\II)= 0.
\]
Therefore, combining this fact with $\E\lones{{\llambda^\eps}^\top(\PP^\pi-\II)}\le \frac{1}{M}\eps$ as we have shown earlier yields 
\[
\E\lone{(\llambda^\eps-\nnu^\pi)^\top(\PP^\pi-\II)}\le \frac{1}{M}\eps.
\]

It also leads to
\begin{align*}
\E\left[(\nnu^\pi-\llambda^\eps)^\top\rr^\pi\right] & = \E \left[(\nnu^\pi-\llambda^\eps)^\top(\rr^\pi-(\langle\rr^\pi,\nnu^\pi\rangle)\1)\right]\\
& =  \E \left[(\nnu^\pi-\llambda^\eps)^\top\left(\II-\PP^\pi+\1(\nnu^\pi)^\top\right)\left(\II-\PP^\pi+\1(\nnu^\pi)^\top\right)^{-1}(\rr^\pi-(\langle\rr^\pi,\nnu^\pi\rangle)\1)\right]\\
& \le \E\lone{(\nnu^\pi-\llambda^\eps)^\top\left(\II-\PP^\pi+\1(\nnu^\pi)^\top\right)}\linf{\left(\II-\PP^\pi+\1(\nnu^\pi)^\top\right)^{-1}(\rr^\pi-(\langle\rr^\pi,\nnu^\pi\rangle)\1)}\\
& \le M\cdot\E\lone{(\nnu^\pi-\llambda^\eps)^\top\left(\II-\PP^\pi\right)}\le \eps,
\end{align*}
where for the last but one inequality we use the definition of $M=2\tmix$ and Lemma~\ref{lem:norm-bounds-mixing}.

Note now the average reward under policy $\pi$ satisfies 
\begin{align*}
\E\bar{v}^\pi =  & \E\left[(\nnu^\pi)^\top\rr^\pi\right] =\E\left[{\nnu^\pi}^\top(\PP^\pi-\II)\vv^*+(\nnu^\pi)^\top\rr^\pi\right] \\
= & \E\left[{(\nnu^\pi-\llambda^\eps)}^\top\left[(\PP^\pi-\II)\vv^*+\rr^\pi\right]\right] + \E\left[{\llambda^\eps}^\top[(\PP^\pi-\II)\vv^*+\rr^\pi]\right]\\
\stackrel{(i)}{\ge} & \E\left[(\nnu^\pi-\llambda^\eps)^\top(\PP^\pi-\II)\vv^*\right] + \E\left[(\nnu^\pi-\llambda^\eps)^\top \rr^\pi\right] + \bar{v}^*-\eps\\
\stackrel{(ii)}{\ge} & \bar{v}^* -\E\lones{{(\nnu^\pi-\llambda^\eps)}^\top(\PP^\pi-\II)}\linf{\vv^*} -\E\left[(\nnu^\pi-\llambda^\eps)^\top \rr^\pi\right]-\eps\\
\stackrel{(iii)}{\ge} &  \bar{v}^*-\frac{1}{M}\eps\cdot M - (\eps\cdot 1)-\eps
= \bar{v}^*-3\eps 
\end{align*}
where we use $(i)$ the optimality relation stated in Lemma~\ref{lem:gap-mixing}, $(ii)$ Cauchy-Schwarz inequality and $(iii)$ conditions on $\ell_1$ bounds of $(\llambda^\eps-\nnu^\pi)^\top(\PP^\pi-\II)$ and $(\llambda^\eps-\nnu^\pi)^\top\rr^\pi$ we prove earlier.
\end{proof}
}

Lemma~\ref{lem:approx-mixing} shows one can construct an expected  $\eps$-optimal policy from an expected $\eps/3$-approximate solution of the minimax problem~\eqref{def-minimax-mixing}. \notarxiv{Thus, using Corollary~\ref{cor:mixing-regret}
    one directly obtains our desired total sample complexity for Algorithm~\ref{alg:sublinear-mixing} to solve mixing AMDPs to desired accuracy, as stated in Theorem~\ref{thm:sub-mixing-main}.} \arxiv{Thus, using Corollary~\ref{cor:mixing-regret}
    one directly gets the total sample complexity for Algorithm~\ref{alg:sublinear-mixing} to solve mixing AMDP to desired accuracy, as stated in Theorem~\ref{thm:sub-mixing-main}. For completeness we restate the theorem include a short proof below.}
 
 \arxiv{

\restatemixingmain*

\begin{proof}[Proof of Theorem~\ref{thm:sub-mixing-main}]
Given a mixing AMDP tuple $\calM=(\calS,\calA,\calP,\RR)$ and $\epsilon \in (0,1)$, one can construct an approximate policy $\pi^\eps$ using Algorithm~\ref{alg:sublinear-mixing} with accuracy level set to $\eps'=\frac{1}{3}\eps$ such that by Lemma~\ref{lem:approx-mixing},
\begin{equation*}
\E\bar{v}^{\pi^\eps}\ge\bar{v}^*-\eps.
\end{equation*}
It follows from Corollary~\ref{cor:mixing-regret} that the sample complexity is bounded by 
\[\O{\frac{\tmix^2\A\log(\A)}{\eps^2}}.\]
\end{proof}

}

\section{Constrained MDP}
\label{sec:constrained}

In this section, we consider solving a generalization of the mixing AMDP problem with additional linear constraints, which has been an important and well-known problem class along the study of MDP~\cite{altman1999constrained}\notarxiv{; we defer readers to Appendix~\ref{app:con} for derivation ommitted in this section}. 

Formally, we focus on approximately solving the following dual formulation of constrained mixing AMDPs~\footnote{One can reduce the general case of $\DD^\top\mmu\ge\cc$ for some $\cc>0$ to this case by taking $\dd_k\leftarrow \dd_k/c_k$, under which an $\eps$-approximate solution as defined in~\eqref{def-eps-con} of the modified problem corresponds to a multiplicatively approximate solution satisfying $\DD^\top\mmu\ge(1-\eps)\cc$.}~: %
\begin{equation}\label{def:conAMDP}
\begin{aligned}
& \text{(D)} & \max_{\mmu\in\Delta^\calA} & & 0 &\\
&  & \text{subject to } & & (\hat{\II}- \PP)^\top\mmu & = \0, \quad \DD^\top \mmu\ge \1,
\end{aligned}
\end{equation}
where $\DD = \begin{bmatrix} \dd_1 & \cdots &\dd_K\end{bmatrix}$ under the additional assumptions that $\dd_k\ge\0, \forall k\in[K]$ and the problem is strictly feasible (with an inner point in its feasible set). Our goal is to compute $\epsilon$-approximate policies and solutions for \eqref{def:conAMDP} defined as follows.

\begin{definition}
Given a policy $\pi$ with its stationary distribution $\nnu^\pi$, it is an $\eps$-approximate policy of system~\eqref{def:conAMDP} if for $\mmu$ defined as $\mu_{i,a_i}=\nu^\pi_i\pi_{i,a_i},\forall i\in\calS,a_i\in\calA_i$ it is an $\eps$-approximate solution of~\eqref{def:conAMDP}, i.e. it satisfies
\begin{align}
\mmu^\top(\hat{\II}-\PP)=\0,\quad \DD^\top\mmu\ge(1-\eps)\1.\label{def-eps-con}
\end{align}
\end{definition}

\arxiv{
By considering (equivalently) the relaxation of~\eqref{def:conAMDP} with $\mmu\ge\0,\lone{\mmu}\le 1$ instead of $\mmu\in\Delta^\calA$, one can obtain the following primal form of the problem:
\begin{equation*}
\begin{aligned}
& \text{(P)} & \min_{\sss\ge\0,\vv,t\ge0} & & t-\sum_k s_k &    \\
&  & \text{subject to } & &   (\PP-\hat{\II}) \vv+ & \DD\sss\le t\1.
\end{aligned}	
\end{equation*}

Now by our assumptions, strong duality and strict complementary slackness there exists some optimal $t^*>0$. Thus we can safely consider the case when $t>0$, and rescale all variables $\sss$, $\vv$, and $t$ by $1/t$ without changing that optimal solution with $t^*>0$ to obtain the following equivalent primal form of the problem: 
\begin{equation*}
\label{def-lp-mixing-matrix-gen}
\begin{aligned}
& \text{(P')} & \min_{\sss\ge\0,\vv} & & 1-\sum_k s_k &    \\
&  & \text{subject to } & &   (\PP-\hat{\II}) \vv+ & \DD\sss\le \1.
\end{aligned}	
\end{equation*}
}

For $D\defeq\norm{\DD}_{\max}\defeq\max_{i,a_i,k}|[d_k]_{i,a_i}|$ and $M\defeq 2D\tmix$ we consider the following equivalent problem:
\arxiv{
\begin{align}
\min_{\vv\in\B_{2M}^\calS,\sss:\sum_{k}s_k\le2, \sss\ge\0}\quad\max_{\mmu\in\Delta^\calA}\quad f(\vv,\sss,\mmu)\defeq\mmu^\top\left[(\hat{\II}-\PP)\vv+\DD\sss\right]-\1^\top\sss.\label{def-minimax-con}
\end{align}
}
\notarxiv{
\begin{align}
\min_{\vv\in\B_{2M}^\calS,\sss:\sum_{k}s_k\le2, \sss\ge\0}\quad\max_{\mmu\in\Delta^\calA}\quad f(\vv,\sss,\mmu)\label{def-minimax-con}\\
\text{where }f(\vv,\sss,\mmu)\defeq\mmu^\top\left[(\hat{\II}-\PP)\vv+\DD\sss\right]-\1^\top\sss.\nonumber
\end{align}
}
Note in the formulation we pose the additional constraints on $\vv$, $\sss$ for the sake of analysis. These constraints don't change the problem optimality by noticing $\vv^*\in\ball_{2M}^\calS$, $\sss^*\in\Delta^K$
\notarxiv{
; see Appendix~\ref{app:con} for details.
}
\arxiv{
. More concretely for $\sss^*$, due to the feasibility assumption and strong duality theory, we know the optimality must achieve when $1-\sum_k s^*_k=0$, i.e. one can safely consider a domain of $\sss$ as $\sum_k s_k\le 2, \sss\ge\0$. For the bound on $\vv^*$, using a method similar as in Section~\ref{ssec:bound-mixing} we know there exists some $\vv^*$, optimal policy $\pi$, its corresponding stationary distribution $\nnu^\pi$ and probability transition matrix $\PP^\pi$ satisfying 
\[
(\PP^\pi-\hat{\II})\vv^*+\DD\sss^*=\rr^*\le \1,\]
which implies that 
\[
\exists \vv^*\perp\nnu^\pi, \norm{\vv^*}_\infty = \norm{(\II-\PP^\pi+\1(\nnu^\pi)^\top)^{-1}(\DD\sss^*-\rr^*)}_\infty\le 2D\tmix,
\]
where the last inequality follows from Lemma~\ref{lem:norm-bounds-mixing}.
}

\arxiv{
To solve the problem we again consider a slight variant of the framework in Section~\ref{sec:framework}. We work with the new spaces induced and therefore use new Bregman divergences as follows:

We set Bregman divergence unchanged with respect to $\mmu$ and $\vv$, for $\sss$,  we consider a standard distance generating function for $\ell_1$ setup defined as $r(\sss)\defeq\sum_{k}s_k\log(s_k)$, note it induces a rescaled KL-divergence as $V_{\sss}(\sss')\defeq\sum_{k}s_k\log(s_k'/s'_k)-\lones{\sss}+\lones{\sss'}$, which also satisfies the local-norm property that
\[
\forall \sss',\sss\ge\0, \1^\top\sss\le2, \1^\top\sss'\le2, k\ge6, \linf{\delta}\le1; \langle \delta,\sss'-\sss\rangle-V_{\sss'}(\sss) \le \sum_{k\in[K]}s_k\delta_k^2.
\]

Now for the primal side, the gradient mapping is $g\vsf(\vv,\sss,\mmu)=(\hat{\II}-\PP)^\top\mmu$, $g\ssf(\vv,\sss,\mmu)=\DD^\top\mmu-\1$, we can define gradient estimators correspondingly as

\begin{equation}
\label{eq:est-vs-mixing}
\begin{aligned}
\text{Sample } & (i,a_i)\sim [\mmu]_{i,a_i}, j\sim p_{ij}(a_i),\quad & \text{set } & \tilde{g}\vsf(\vv,\sss,\mmu) = \ee_j-\ee_i.\\
\text{Sample } & (i,a_i)\sim [\mmu]_{i,a_i}, k\sim 1/K,\quad & \text{set } & \tilde{g}\ssf(\vv,\sss,\mmu) = K[d_k]_{i,a_i}\ee_k-\1.
\end{aligned}
\end{equation}

These are bounded gradient estimator for the primal side respectively.
\begin{lemma}
\label{lem:est-property-vs-mixing}
$\tilde{g}\vsf$ defined in~\eqref{eq:est-vs-mixing} is a $(2,\norm{\cdot}_2)$-bounded  estimator, and $\tilde{g}\ssf$ defined in~\eqref{eq:est-vs-mixing} is a $(KD+2,2KD^2+2,\norm{\cdot}_{\Delta^K})$-bounded  estimator.
\end{lemma}

For the dual side, $g\musf(\vv,\sss,\mmu) = (\hat{\II}-\PP)\vv+\DD\sss$, with its gradient estimator
\begin{equation}
\label{eq:est-mus-mixing}
\begin{aligned}
\text{Sample } & (i,a_i)\sim 1/\A, j\sim p_{ij}(a_i), k\sim s_k/\lones{\sss}\quad;\\ \text{set } & \tilde{g}\musf(\vv,\sss,\mmu) = \A(v_i-\gamma v_j-r_{i,a_i} + [d_k]_{i,a_i}\lones{\sss})\ee_{i,a_i}.
\end{aligned}
\end{equation}

This is a bounded gradient estimator for the dual side with the following property.
\begin{lemma}
\label{lem:est-property-mus-mixing}
$\tilde{g}\musf$ defined in~\eqref{eq:est-mus-mixing} is a $((2M+1+2D)\A,2(2M+1+2D)^2\A,\norm{\cdot}_{\Delta^\calA})$-bounded estimator.
\end{lemma}

\begin{algorithm}
\caption{SMD for generalized saddle-point problem~\eqref{def-minimax-con}}
	\label{alg:framework-gen}
\begin{algorithmic}[1]
 \STATE \textbf{Input:} Desired accuracy $\eps$.
 \STATE	\textbf{Output:} An expected $\eps$-approximate solution $(\vv^\eps,\sss^\eps, \mmu^\eps)$ for problem~\eqref{def-minimax-con}.
\STATE \textbf{Parameter:} Step-size $\eta\vsf = O(\eps)$, $\eta\ssf = O(\eps K^{-1}D^{-2})$, $\eta\musf = O(\eps\tmix^{-2}D^{-2}\A^{-1})$, total iteration number $T\ge \Theta((\tmix^2\A+K)D^2\eps^{-2}\log(\A))$.
		\FOR{$t=1,\ldots,T-1$}
			\STATE
			 Get $\tilde{g}\vsf_t$ as a bounded estimator of $g\vsf(\vv_t,\sss_t,\mmu_t)$
			 \STATE Get $\tilde{g}\ssf_t$  as a bounded estimator for $g\ssf(\vv_t,\sss_t,\mmu_t)$
			 \STATE Get $\tilde{g}\musf_t$  as a bounded estimator for $g\musf(\vv_t,\sss_t,\mmu_t)$
			\STATE Update $\vv_{t+1} \leftarrow \argmin\limits_{\vv\in \B_{2D\tmix}^\calS} \langle \eta\vsf \tilde{g}\vsf_t, \vv \rangle + V_{\vv_{t}}(\vv)$
			\STATE Update $\sss_{t+1} \leftarrow \argmin\limits_{\sss\ge\0,\sum_k s_k\le 2} \langle \eta\ssf \tilde{g}\ssf_t, \sss \rangle + V_{\sss_{t}}(\sss)$
			\STATE Update $\mmu_{t+1} \leftarrow \argmin\limits_{\mmu\in\Delta^\calA} \langle \eta\musf \tilde{g}\musf_t, \mmu \rangle + V_{\mmu_{t}}(\mmu)$
		\ENDFOR
		\STATE \textbf{Return} $(\vv^\eps,\sss^\eps,\mmu^\eps)\leftarrow\frac{1}{T}\sum_{t\in[T]} (\vv_t,\sss_t,\mmu_t)$
\end{algorithmic}
\end{algorithm}

Given the guarantees of designed gradient estimators in Lemma~\ref{lem:est-property-vs-mixing},~\ref{lem:est-property-mus-mixing} and choice of $M=2D\tmix$, one has the following Algorithm~\ref{alg:framework-gen} for finding an expected $\eps$-optimal solution of minimax problem~\eqref{def-minimax-con}, with its theoretical guarantees as stated in Theorem~\ref{cor:mixing-regret}.
}

\notarxiv{By designing gradient estimators and choosing divergence terms properly, one can obtain an approximately optimal solution efficiently, and thus an  approximately optimal policy.}

\begin{theorem}
\label{cor:mixing-regret}
Given mixing AMDP tuple $\calM=(\calS,\calA,\calP,\RR)$ with constraints $D\defeq\max_{i,a_i,k}|[d_k]_{i,a_i}|$, for accuracy $\epsilon \in (0,1)$, Algorithm~\ref{alg:framework-gen} with gradient estimators~\eqref{eq:est-vs-mixing},~\eqref{eq:est-mus-mixing} and parameter choice $\eta\vsf=O(\eps)$, $\eta\ssf=O(\eps K^{-1}D^{-2})$, $\eta\musf=O(\eps\tmix^{-2}D^{-2}\A^{-1})$ outputs an expected $\eps$-approximate solution to constrained mixing minimax problem~\eqref{def-minimax-con} 
with sample complexity $O({(\tmix^2\A+K)D^2}{\eps^{-2}}\log(K\A))$.
\end{theorem}

Due to the similarity to Theorem~\ref{thm:framework}, here we only provide a proof sketch capturing the main steps and differences within the proof.

\begin{proof-sketch}
\ \\
\emph{Regret bounds with local norms.} The core statement is a standard regret bound using local norms (see Lemma~\ref{lem:mirror-descent-l2} for $\vv$ and Lemma~\ref{lem:mirror-descent-l1}) for $\sss$ and $\mmu$, which summing together gives the following guarantee \notarxiv{(let $\tilde{g}\x_t,\tilde{g}\y_t$ denote $\tilde{g}\x(\xx_t,\yy_t),\tilde{g}\y(\xx_t,\yy_t)$)}
\begin{equation}\label{eq:regret-sketch-con}
	\begin{aligned}
		& \sum_{t\in[T]}\langle\tilde{g}\vsf_t, \vv_t-\vv\rangle +\sum_{t\in[T]}\langle \tilde{g}\ssf_t, \sss_t-\sss\rangle + \sum_{t\in[T]}\langle\tilde{g}\musf_t,\mmu_t-\mmu\rangle\\
		\le &  \frac{V_{\vv_1}(\vv)}{\eta\vsf} +\frac{\sum\limits_{t\in[T]}{\eta\vsf}\|\tilde{g}\vsf_t\|_2^2}{2}+\frac{V_{\sss_1}(\sss)}{\eta\ssf} +\frac{\sum\limits_{t\in[T]}\eta\ssf\norm{\tilde{g}\ssf_t}_{\sss_t}^2}{2} +\frac{V_{\mmu_1}(\mmu)}{\eta\musf}+\frac{\sum\limits_{t\in[T]}\eta\musf\norm{\tilde{g}\musf_t}_{\mmu_t}^2}{2}.
	\end{aligned}
\end{equation}
Note one needs the bounded maximum entry condition for $\tilde{g}\ssf$, $\tilde{g}\musf$, and the fact that rescaled KL-divergence also satisfies local-norm property in order to use Lemma~\ref{lem:mirror-descent-l1}.\\
\emph{Domain size.} The domain size can be bounded as \[\max_{\vv\in \B_{2D\tmix}^\calS}V_{\vv_1}(\vv)\le O(|\calS|D^2\tmix^2), \quad \max_{\sss\ge\0,\sum_k s_k\le 2}V_{\sss_1}(\sss)\le O(\log K),\quad \max_{\mmu\in\Delta^\calA}V_{\mmu_1}(\mmu)\le O(\log \A) \] by definition of their corresponding Bregman divergences.\\
\emph{Second-moment bounds.}
This is given through the bounded second-moment properties of estimators directly, as in Lemma~\ref{lem:est-property-vs-mixing} and~\ref{lem:est-property-mus-mixing}.\\
\emph{Ghost-iterate analysis.}
In order to substitute $\tilde{g}\vsf,\tilde{g}\ssf,\tilde{g}\musf$ with $g\vsf,g\ssf,g\musf$ for LHS of Eq.~\eqref{eq:regret-sketch-con}, one can apply the regret bounds again to ghost iterates generated by taking gradient step with $\hat{g}=g-\tilde{g}$ coupled with each iteration. The additional terms coming from this extra regret bounds are in expectation $0$ through conditional expectation computation.\\
\emph{Optimal tradeoff.} One pick $\eta\vsf,\eta\ssf,\eta\musf,T$ accordingly to get the desired guarantee as stated in Theorem~\ref{cor:mixing-regret}.
\end{proof-sketch}

\notarxiv{
Following the similar rounding technique as in Section~\ref{sec:policy}, one can show the approximate solution $\mmu^\eps$ for minimax problem~\eqref{def-minimax-con} leads to an approximately optimal policy $\pi^\eps$ an  of problem~\eqref{def:conAMDP}.
}

Similar to Section~\ref{ssec:sub-mixing}, one can round an $\eps$-optimal solution to an $O(\eps)$-optimal policy utilizing the policy obtained from the unique decomposition of $\mmu^\eps$.

\begin{corollary}
\label{cor:mixing-policy}
Following the setting of Corollary~\ref{cor:mixing-regret}, the policy $\pi^\eps$ induced by the unique decomposition of $\mmu^\eps$ from the output satisfying $\mu^\eps_{i,a_i} = \lambda^\eps_{i}\cdot \pi^\eps_{i,a_i}$, is an $O(\eps)$-approximate policy for system~\eqref{def:conAMDP}.
\end{corollary}

\arxiv{
\begin{proof}[Proof of Corollary~\ref{cor:mixing-policy}]
Following the similar rounding technique as in Section~\ref{ssec:sub-mixing}, one can consider the policy induced by the $\eps$-approximate solution of MDP $\pi^\eps$ from the unique decomposition of  $\mu^\eps_{i,a_i} = \lambda^\eps_{i}\cdot \pi^\eps_{i,a_i}$, for all $i\in\calS, a_i\in\calA_i$.

Given the optimality condition, we have
$$\E\left[f(\vv^*,\sss^\eps,\mmu^\eps)-\min_{\vv\in\calV}f(\vv,\sss^\eps,\mmu^\eps)\right]\le\eps,$$
which is also equivalent to (denoting $\pi\defeq\pi^\eps$, and $\nnu^\pi$ as stationary distribution under it)
\[
\E\left[\max_{\vv\in\calV}{\llambda^\eps}^\top(\II-\PP^\pi)(\vv^*-\vv)\right]\le \eps,
\]
thus implying that $\lones{(\llambda^\eps)^\top(\II-\PP^\pi)}\le\frac{1}{M}\eps$, $\lones{(\llambda^\eps-\nnu^\pi)^\top(\II-\PP^\pi-\1(\nnu^\pi)^\top)}\le\frac{1}{M}\eps = O(\tfrac{1}{\tmix D}\eps)$ hold in expectation.

Now consider $\mmu$ constructed from $\mu_{i,a_i}=\nu^\eps_{i}\cdot\pi_{i,a_i}^\eps$, by definition of $\nnu$ it holds that $\mmu(\hat{\II}-\PP)=0$. 

For the second inequality of problem~\eqref{def:conAMDP}, similarly in light of primal-dual optimality
\[
\E\left[f(\vv^\eps,\sss^*,\mmu^\eps)-\min_{\sss\ge\0:\sum_k s_k\le 2}f(\vv^\eps,\sss,\mmu^\eps)\right]\le\eps \quad\Leftrightarrow\quad\E\left[\max_{\sss\ge\0:\sum_k s_k\le 2}\left[(\mmu^\eps)^\top\DD-\1^\top\right]\left(\sss^*-\sss\right)\right]\le \eps,
\]
which implies that $\DD^\top\mmu^\eps\ge \ee-\eps\1$ hold in expectation given $\sss^*\in\Delta^K$.

Consequently, we can bound the quality of dual variable $\mmu$
\begin{align*}
\DD^\top\mmu & = \DD^\top\mmu^\eps+\DD^\top(\mmu-\mmu^\eps) =  \DD^\top\mmu^\eps+\DD^\top{\Pi^\eps}^\top(\nnu^\pi-\llambda^\eps)	\\
& = \DD^\top\mmu^\eps+\DD^\top{\Pi^\eps}^\top(\II-(\PP^\pi)^\top+\nnu^\pi\1^\top)^{-1}(\II-(\PP^\pi)^\top+\nnu^\pi\1^\top)(\nnu^\pi-\llambda^\eps)\\
& \ge \ee-\eps\1-\linf{\DD^\top{\Pi^\eps}^\top(\II-(\PP^\pi)^\top+\nnu^\pi\1^\top)^{-1}(\II-(\PP^\pi)^\top+\nnu^\pi\1^\top)(\nnu^\pi-\llambda^\eps)}\cdot\1\\
& \ge \ee-\eps\1-\max_{k}\linf{\dd_k^\top{\Pi^\eps}^\top(\II-(\PP^\pi)^\top+\nnu^\pi\1^\top)^{-1}}\lone{(\II-(\PP^\pi)^\top+\nnu^\pi\1^\top)(\nnu^\pi-\llambda^\eps)}\cdot\1\\
& \ge \ee-\O{\eps}\1,
\end{align*}
where the last inequality follows from definition of $D$, $\Pi^\eps$, Lemma~\ref{lem:norm-bounds-mixing} and the fact that \[\lones{(\llambda^\eps-\nnu^\pi)^\top(\II-\PP^\pi-\1(\nnu^\pi)^\top)}\le O(\tfrac{\eps}{\tmix D}).\]

From above we have shown that assuming the stationary distribution under $\pi^\eps$ is $\nnu^\eps$, it satisfies $\lones{\nnu^{\eps}-\llambda^\eps}\le O(\tfrac{\eps}{D\tmix})$, thus giving an approximate solution $\mmu = \nnu^{\eps}\cdot\pi^{\eps}$ satisfying $\mmu(\hat{\II}-\PP)=0$, $\DD^\top \mmu\ge \ee-O(\eps)$ and consequently for the original problem~\eqref{def:conAMDP} an approximately optimal policy $\pi^\eps$.
\end{proof}
}

\section{Conclusion}

This work offers a general framework based on stochastic mirror descent to find an $\eps$-optimal policy for AMDPs and DMDPs. It offers new insights over previous SMD approaches for solving MDPs, achieving a better sample complexity and removing an ergodicity condition for mixing AMDP, while matching up to logarithmic factors the known SMD method for solving DMDPs. 
\notarxiv{

This work reveals an interesting connection MDP problems and $\ell_\infty$-regression. We believe there are a number of interesting directions and open problems for future work, including getting optimal sample complexity for discounted case, obtaining high-precision algorithms, extending the framework to broader classes of MDPs, etc.  See Appendix~\ref{app:conclusion} for a more detailed discussion of these open directions. We hope a better understanding of these problems could lead to a more complete picture of solving MDP and RL using convex-optimization methods.

}
\arxiv{
This work reveals an interesting connection between MDP problems with $\ell_\infty$ regression and opens the door to future research. Here we discuss a few interesting directions and open problems:

\emph{Primal-dual methods with optimal sample-complexity for DMDPs. } For DMDPs, the sample complexity of our method (and the one achieved in~\cite{cheng2020reduction}) has $(1-\gamma)^{-1}$ gap with the known lower bound~\citep{AMK12}, which can be achieved by stochastic value-iteration~\citep{SWWYY18} or $Q$-learning~\citep{W19}. If it is achievable using convex-optimization lies at the core of further understanding the utility of convex optimization methods relative to standard value / policy-iteration methods. 

\emph{High-precision methods.} There have been recent high-precision stochastic value-iteration algorithms~\citep{SWWY18} that produce an $\eps$-optimal strategy in runtime $\Otil{|\calS|\A+(1-\gamma)^{-3}\A}$ while depending logarithmically on $1/\eps$. These algorithms iteratively shrink the value domain in an $\ell_\infty$ ball; it is an interesting open problem to generalize our methods to have this property or match this runtime. 

\emph{Lower bound for AMDPs.} There has been established lower-bound on sample complexity needed for DMDP~\citep{AMK12}, however the lower bound for average-reward MDP is less understood. For mixing AMDP, we ask the question of what the best possible sample complexity dependence on mixing time is, and what the hard cases are. For more general average-reward MDP, we also ask if there is any lower-bound result depending on problem parameters other than mixing time.

\emph{Extension to more general classes of MDP.} While average-reward MDP with bounded mixing time $\tmix$ and DMDP with discount factor $\gamma$ are two fundamentally important classes of MDP, there are instances that fall beyond the range. It is thus an interesting open direction to extend our framework for more general MDP instances and understand what problem parameters the sample complexity of SMD-like methods should depend on.
}

\section*{Acknowledgements}
This research was partially supported by NSF CAREER Award CCF-1844855, a PayPal research gift, and a Stanford Graduate Fellowship. We thank Yair Carmon and Kevin Tian for helpful discussions on coordinate methods for matrix games; we thank Mengdi Wang, Xian Wu, Lin F. Yang, and Yinyu Ye for helpful discussions regarding MDPs; we thank Ching-An Cheng, Remi Tachet des Combes, Byron Boots, Geoff Gordon for pointing out their paper to us; we also thank the anonymous reviewers who helped
improve the completeness and readability of this paper by providing many helpful comments.

\arxiv{
\bibliographystyle{abbrvnat}
\bibliography{references.bib}
}
\notarxiv{
\bibliography{references.bib}
\bibliographystyle{icml2020}
}

\onecolumn
\newpage
\appendix

\part*{Supplementary material}
\section*{Appendix}

\section{DMDPs}
\label{sec:DMDP}

In this section we provide the corresponding sample complexity results for DMDPs to formally prove Theorem~\ref{thm:sub-discounted-main}. In Section~\ref{ssec:bound-discounted} we specify the choice of $M$ in minimax problem~\eqref{def-minimax-discounted} by bounding the operator norm to give a domain that $\vv^*$ lies in. In Section~\ref{ssec:est-discounted} we give estimators for both sides for solving~\eqref{def-minimax-discounted}, which is similar to the estimators developed in Section~\ref{sec:framework}. In Section~\ref{ssec:sub-discounted} we show how to round an $\eps$ optimal solution of~\eqref{def-minimax-discounted} to an $\eps$ optimal policy. 

\subsection{Bound on Matrix Norm}
\label{ssec:bound-discounted}

For discounted case, we can alternatively show an upper bound on matrix norm using discount factor $\gamma$, formally stated in Lemma~\ref{lem:norm-bounds-discounted}, and used for definition of $M$ and proof of Lemma~\ref{lem:approx-discounted} in Section~\ref{ssec:sub-discounted}.

\begin{restatable}{lemma}{restateNormDiscounted}
\label{lem:norm-bounds-discounted}
Given a DMDP with discount factor $\gamma\in(0,1)$, for any probability transition matrix $\PP^\pi\in\R^{\calS\times \calS}$ under certain policy $\pi$, it holds that $(\II-\gamma\PP^\pi)^{-1}$ is invertible with
\[\linfs{(\II-\gamma\PP^\pi)^{-1}}\le\frac{1}{1-\gamma}.\]
\end{restatable}

\begin{proof}[Proof of Lemma~\ref{lem:norm-bounds-discounted}] 
	
First, we claim that 
\begin{equation}
\label{eq:coordinate-lower-bound}
\min_{v \in \R^\calS : \norm{v}_\infty = 1} \norm{(\II - \gamma \PP^\pi)^{-1} v} \geq 1 - \gamma ~.
\end{equation}
To see this, let $\vv \in \R^\calS$ with $\norm{\vv}_\infty = 1$ be arbitrary and let $i \in \calS$ be such that $|v_i| = 1$. We have 
\begin{align*}
\left|[(\II-\gamma\PP^\pi) \vv]_i\right| & = \left| v_i-\gamma \sum_{j \in \calS}\PP^\pi(i,j) v_j\right|
 \ge \left|v_i\right|-\left|\gamma \sum_{j \in \calS} \PP^\pi(i,j) v_j\right|\\
& \ge 1-\gamma \sum_{j \in \calS} \PP^\pi(i,j)| v_j|\ge 1-\gamma.
\end{align*}
Applying the claim yields the result as \eqref{eq:coordinate-lower-bound} implies invertibility of $\II-\gamma\PP^\pi$ and 
\begin{align*}
	\linf{(\II-\gamma\PP^\pi)^{-1}} & \defeq\max_{\vv\in\R^\calS}\frac{\linf{(\II-\gamma\PP^\pi)^{-1}\vv}}{\linf{\vv}} \\
	& \stackrel{(i)}{=} \max_{\hat{\vv}}\frac{\linf{(\II-\gamma\PP^\pi)^{-1}(\II-\gamma\PP^\pi)\hat{\vv}}}{\linf{(\II-\gamma\PP^\pi)\hat{\vv}}}\\
	& \stackrel{(ii)}{=}\max_{\hat{\vv}:\linf{\hat{\vv}}=1}\frac{\linf{(\II-\gamma\PP^\pi)^{-1}(\II-\gamma\PP^\pi)\hat{\vv}}}{\linf{(\II-\gamma\PP^\pi)\hat{\vv}}}\\
	& =\frac{1}{\min_{\hat{\vv}:\linf{\hat{\vv}}=1}\linf{(\II-\gamma\PP^\pi)\hat{\vv}}},
\end{align*}
where in $(i)$ we replaced $\vv$ with $(\II-\gamma\PP^\pi)\hat{\vv}$ for some $\hat{\vv}$ since $\II-\gamma\PP^\pi$ is invertible and in $(ii)$ we rescaled $\hat{\vv}$ to satisfy $\linf{\hat{\vv}}=1$ as scaling $\hat{\vv}$ does not affect the ratio so long as $\hat{\vv} \neq 0$.
\end{proof}

\begin{corollary}[Bound on $\vv^*$]\label{cor:bound-discounted}
For DMDP~\eqref{def-Bellman-discounted}, the optimal value vector $\vv^*$ satisfies
\[
\linf{\vv^*}\le (1-\gamma)^{-1}.
\]	
\end{corollary}

\begin{proof}[Proof of Corollary~\ref{cor:bound-discounted}]
	By optimality conditions and Lemma~\ref{lem:norm-bounds-discounted} one has $\linf{\vv^*}=\linf{(\II-\gamma\PP^{*})^{-1}\rr^{*}}\le (1-\gamma^{-1})$. 
\end{proof}

Thus, we can safely consider the minimax problem~\eqref{def-minimax-discounted} with the additional constraint $\vv\in\calV$, where we set $M = (1-\gamma)^{-1}$. The extra coefficient $2$ comes in to ensure stricter primal-dual optimality conditions, which we use in Lemma~\ref{lem:approx-discounted} for the rounding.

\subsection{Design of Estimators}
\label{ssec:est-discounted}

Given $M=(1-\gamma)^{-1}$, for discounted case one construct gradient estimators in a similar way. For the $\vv$-side, we consider the following gradient estimator
\begin{equation}
\label{eq:est-v-discounted}
\begin{aligned}
\text{Sample } & (i,a_i)\sim [\mmu]_{i,a_i}, j\sim p_{ij}(a_i), i'\sim q_{i'}\\
\text{Set } & \tilde{g}\vsf(\vv,\mmu) = (1-\gamma)\ee_{i'} + \gamma\ee_j-\ee_i.
\end{aligned}
\end{equation}

\begin{lemma}
\label{lem:est-property-v-discounted}
\arxiv{$\tilde{g}\vsf$ defined in ~\eqref{eq:est-v-discounted} is a $(2,\norm{\cdot}_2)$-bounded  estimator.}
\notarxiv{$\tilde{g}\vsf$ as in ~\eqref{eq:est-v-discounted} is a $(2,\norm{\cdot}_2)$-bounded  estimator.}
\end{lemma}

\begin{proof}[Proof of Lemma~\ref{lem:est-property-v-discounted}]
For unbiasedness, one compute directly that
\begin{align*}
\E\left[\tilde{g}\vsf{(\vv,\mmu)}\right] =(1-\gamma)\qq + \sum_{i,a_i,j}\mu_{i,a_i}p_{ij}(a_i)(\gamma\ee_j-\ee_i)=(1-\gamma)\qq+\mmu^\top(\gamma\PP-\hat{\II}).
\end{align*}
For bound on second-moment, note $\ltwos{\tilde{g}\vsf{(\vv,\mmu)}}^2\le2$ with probability 1 by definition and the fact that $\qq\in\Delta^\calS$, the result follows immediately.
\end{proof}

For the $\mmu$-side, we consider the following gradient estimator
\begin{equation}
\label{eq:est-mu-discounted}
\begin{aligned}
\text{Sample } & (i,a_i)\sim\frac{1}{\A}, j\sim p_{ij}(a_i).\\
\text{Set } & \tilde{g}\musf{(\vv,\mmu)} =  \A(v_i-\gamma v_j-r_{i,a_i})\ee_{i,a_i}.
\end{aligned}
\end{equation}

\begin{lemma}
\label{lem:est-property-mu-discounted}
$\tilde{g}\musf$ defined in ~\eqref{eq:est-mu-discounted} is a $((2M+1)\A,9(M^2+1)\A,\norm{\cdot}_{\Delta^\calA})$-bounded  estimator.
\end{lemma}

\begin{proof}[Proof of Lemma~\ref{lem:est-property-mu-discounted}]
For unbiasedness, one compute directly that
\begin{align*}
\E\left[\tilde{g}\musf{(\vv,\mmu)}\right] & = \sum_{i,a_i}\sum_{j}p_{ij}(a_i)(v_i-\gamma v_j-r_{i,a_i})\ee_{i,a_i} =(\hat{\II}-\gamma\PP)\vv-\rr.
\end{align*}
For bound on $\ell_\infty$ norm, note that with probability 1 we have $\linf{\tilde{g}\musf(\vv,\mmu)}\le (2M+1)\A$ given $|v_i-\gamma\cdot v_j-r_{i,a_i}|\le\max\{2M,\gamma\cdot 2M+1\}\le 2M+1$ by domain bounds on $\vv$.
For bound on second-moment, for any $\mmu'\in\calU$ we have %
\begin{align*}
\E[\norm{\tilde{g}\musf{(\vv,\mmu)}}_{\mmu'}^2] \le \sum_{i,a_i}\frac{1}{\A}\mu_{i,a_i}'\left\{(2M)^2,(2M+1)^2\right\}\A^2 \le 9 (M^2+1)\A,
\end{align*}
where the first inequality follows by directly bounding $|v_i-\gamma v_j-r_{i,a_i}|\le \max\{2M,\gamma\cdot2M+1\},\forall i,j,a_i$.
\end{proof}

Theorem~\ref{thm:framework} together with guarantees of gradient estimators in use in Lemma~\ref{lem:est-property-v-discounted},~\ref{lem:est-property-mu-discounted} and choice of $M=(1-\gamma)^{-1}$ gives Corollary~\ref{cor:discounted-regret}.

\begin{corollary}
\label{cor:discounted-regret}
Given DMDP tuple $\calM=(\calS,\calA,\calP,\RR,\gamma)$ with desired accuracy $\epsilon \in (0,1)$, Algorithm~\ref{alg:sublinear-mixing} outputs an expected $\eps$-approximate solution to discounted minimax problem~\eqref{def-minimax-discounted} 
with sample complexity \arxiv{\[O({(1-\gamma)^{-2}\A}{\eps^{-2}}\log(\A)).\]}
\notarxiv{$O({(1-\gamma)^{-2}\A}{\eps^{-2}}\log(\A)).$}
\end{corollary}

\subsection{Rounding to Optimal Policy}
\label{ssec:sub-discounted}

Now we proceed to show how to convert an expected $\eps$-approximate solution of~\eqref{def-minimax-discounted} to an expected $\Theta((1-\gamma)^{-1}\eps)$-approximate policy for the dual problem (D) of discounted case~\eqref{def-lp-discounted-matrix}. First we introduce a lemma similar to Lemma~\ref{lem:gap-mixing} that relates the dual variable $\mmu^\eps$ with optimal cost-to-go values $\vv^*$ under $\eps$-approximation.

\begin{lemma}
\label{lem:gap-discounted}
If $(\vv^\eps,\mmu^\eps)$ is an $\eps$-approximate optimal solution to the DMDP minimax problem~\eqref{def-minimax-discounted}, then for optimal $\vv^*$,
\begin{align*}
	\E{\mmu^\eps}^\top\left[(\hat{\II}-\gamma\PP)\vv^*-\rr\right]\le\eps.
\end{align*}
\end{lemma}

\begin{proof}[Proof of Lemma~\ref{lem:gap-discounted}]
Note by definition
\begin{align*}
	\eps\ge \E\Gap (\vv^\eps,\mmu^\eps) \defeq \E\max_{\hat{\vv},\hat{\mmu}}\biggl[(1-\gamma)\qq^\top \vv^\eps+\hat{\mmu}^\top((\gamma\PP-\hat{\II})\vv^\eps+\rr) - (1-\gamma)\qq^\top\hat{\vv}-{\mmu^\eps}^\top((\gamma\PP-\hat{\II})\hat{\vv}+\rr))\biggr].
\end{align*}
When picking $\hat{\vv}=\vv^*,\hat{\mmu}=\mmu^*$ optimizers of the minimax problem, this inequality becomes 
\begin{align*}
	\eps & \ge \E\biggl[(1-\gamma)\qq^\top \vv^\eps+{\mmu^*}^\top((\gamma\PP-\hat{\II})\vv^\eps+\rr) - (1-\gamma)\qq^\top\vv^*-{\mmu^\eps}^\top((\gamma\PP-\hat{\II})\vv^*+\rr)\biggr]\\
	 & \stackrel{(i)}{=} {{\mmu}^*}^\top\rr-(1-\gamma)\qq^\top\vv^*-\E\left[{\mmu^\eps}^\top((\gamma\PP-\hat{\II})\vv^*+\rr)\right]\\
	 & \stackrel{(ii)}{=} \E\left[{\mmu^\eps}^\top\left((\hat{\II}-\gamma\PP)\vv^*-\rr\right)\right],
\end{align*}
where we use $(i)$ the fact that ${\mmu^*}^\top(\II-\gamma\PP)=(1-\gamma)\qq^\top$ by dual feasibility and $(ii)$ $(1-\gamma)\qq^\top\vv^*={\mmu^*}^\top\rr$ by strong duality theory of linear programming.
\end{proof}

Next we transfer an optimal solution to an optimal policy, formally through Lemma~\ref{lem:approx-discounted}.

\begin{lemma}
\label{lem:approx-discounted}
	Given an expected $\eps$-approximate solution $(\vv^\eps,\mmu^\eps)$ for discounted minimax problem as defined in~\eqref{def-minimax-discounted}, let $\pi^\eps$ be the unique decomposition (in terms of $\llambda^\eps$) such that  $\mu^\eps_{i,a_i}=\lambda^\eps_i\cdot\pi^\eps_{i,a_i},\forall i\in \calS,a_i\in \calA_i$, where $\llambda\in\Delta^\calS,\pi^\eps_{i}\in\Delta^{\calA_i},\forall i\in\calS$. Taking $\pi\defeq\pi^\eps$ as our policy, it holds that 
	\notarxiv{$\bar{v}^*\le\E\bar{v}^\pi+3\eps/(1-\gamma).$}
	\arxiv{\[\bar{v}^*\le\E\bar{v}^\pi+3\eps/(1-\gamma).\]}
\end{lemma}

\begin{proof}[Proof of Lemma~\ref{lem:approx-discounted}]
Without loss of generality we reparametrize $(\vv^\eps,\mmu^\eps)$ as an $\eps$-optimal solution in the form $\mu_{i,a_i}^\eps=\lambda_i^\eps\pi_{i,a_i}^\eps$, for some $\llambda^\eps,\pi^\eps$. %
For simplicity we still denote the induced policy as $\pi$ and correspondingly probability transition matrix $\PP^{\pi}$ and expected $\rr^\pi$. 

Given the optimality condition, we have
$$\E\left[f(\vv^*,\mmu^\eps)-\min_{\vv\in\calV}f(\vv,\mmu^\eps)\right]\le\eps,$$
which is also equivalent to
$$
\E\max_{\vv\in\calV}\left[(1-\gamma)\qq^\top+{\llambda^\eps}^\top(\gamma\PP^\pi-\II)\right](\vv^*-\vv)\le \eps.
$$
Notice $\vv\in\calV$, we have $\lones{(1-\gamma)\qq+{\llambda^\eps}^\top(\gamma\PP^\pi-\II)}\le \frac{\epsilon}{M}$ as a consequence of

\begin{align*}
	& 2M \cdot \E\lone{(1-\gamma)\qq+{\llambda^\eps}^\top(\gamma\PP^\pi-\II)}\\
	= & 
	\E\ \left[
	\max_{v \in \calV} \left[(1-\gamma)\qq+{\llambda^\eps}^\top(\gamma\PP^\pi-\II)\right](- \vv)
	\right] \\
	= & 
	\E\ \left[
	\max_{v \in \calV} \left[(1-\gamma)\qq+{\llambda^\eps}^\top(\gamma\PP^\pi-\II)\right] (\vv^* - \vv)
	- \left[(1-\gamma)\qq+{\llambda^\eps}^\top(\gamma\PP^\pi-\II)\right]\vv^*
	\right] \\
	\leq & 
	\epsilon + \E\lone{(1-\gamma)\qq+{\llambda^\eps}^\top(\gamma\PP^\pi-\II)} \norm{\vv^*}_\infty
	\leq \epsilon + M \cdot \E\lone{(1-\gamma)\qq+{\llambda^\eps}^\top(\gamma\PP^\pi-\II)}.
\end{align*}

Now by definition of $\nnu^\pi$ as the dual feasible solution under policy $\pi\defeq \pi^\eps$, 
\[
\E\left[(1-\gamma)\qq^\top+{\nnu^\pi}^\top(\gamma\PP^\pi-\II)\right]= 0.
\]
Combining the two this gives
\[
\E\lone{(\llambda^\eps-\nnu^\pi)^\top(\gamma\PP^\pi-\II)}\le \frac{\eps}{M},
\]
and consequently
\begin{align*}
\E\lone{\llambda^\eps-\nnu^\pi} & = \E\lone{(\gamma\PP^\pi-\II)^{-\top}(\gamma\PP-\II)^\top(\llambda^\eps-\nnu^\pi)}\\
& \le \E\lone{(\gamma\PP^\pi-\II)^{-\top}}\lone{(\gamma\PP-\II)^\top(\llambda^\eps-\nnu^\pi)}\le \frac{M}{M}\eps=\eps,
\end{align*}
where the last but one inequality follows from the norm equality that $\lone{(\gamma\PP^\pi-\II)^{-\top}}=\linf{(\gamma\PP^\pi-\II)^{-1}}$ and Lemma~\ref{lem:norm-bounds-discounted}. Note now the discounted reward under policy $\pi$ satisfies
\begin{align*}
\E(1-\gamma)\bar{v}^\pi =\E(\nnu^\pi)^\top\rr^\pi = &\E\left[(1-\gamma)\qq^\top+{\nnu^\pi}^\top(\gamma\PP^\pi-\II)\right]\vv^*+\E(\nnu^\pi)^\top\rr^\pi \\
= & (1-\gamma)\qq^\top\vv^*+\E\left[{\nnu^\pi}^\top\left[(\gamma\PP^\pi-\II)\vv^*+\rr^\pi\right]\right]\\
=  & (1-\gamma)\qq^\top\vv^*+\E\left[{(\nnu^\pi-\llambda^\eps)}^\top\left[(\gamma\PP^\pi-\II)\vv^*+\rr^\pi\right]\right] +\E\left[{\llambda^\eps}^\top[(\gamma\PP^\pi-\II)\vv^*+\rr^\pi]\right]\\
\stackrel{(i)}{\ge} & (1-\gamma)\qq^\top\vv^* +\E\left[{(\nnu^\pi-\llambda^\eps)}^\top\left[(\gamma\PP^\pi-\II)\vv^*+\rr^\pi\right]\right]-\eps\\
\stackrel{(ii)}{\ge} & (1-\gamma)\bar{v}^* - \E\lones{{(\nnu^\pi-\llambda^\eps)}^\top(\gamma\PP^\pi-\II)}\linf{\vv^*} -\E\lones{\nnu^\pi-\llambda^\eps}\linf{\rr^\pi}-\eps\\
\stackrel{(iii)}{\ge} & (1-\gamma)\bar{v}^*-\frac{1}{M}\eps\cdot M-\frac{M}{M}\eps\cdot1-\eps =(1-\gamma) \bar{v}^*-3\eps,
\end{align*}
where we use $(i)$ the optimality relation stated in Lemma~\ref{lem:gap-discounted}, $(ii)$ Cauchy-Schwarz inequality and $(iii)$ conditions on $\ell_1$ bounds of $(\llambda^\eps-\nnu^\pi)^\top(\gamma\PP^\pi-\II)$ and $\llambda^\eps-\nnu^\pi$ we prove earlier.
\end{proof}

Lemma~\ref{lem:approx-discounted} shows it  suffices to find an expected $(1-\gamma)\eps$-approximate solution to problem~\eqref{def-minimax-discounted} to get an expected $\eps$-optimal policy. Together with Corollary~\ref{cor:discounted-regret} this directly yields the sample complexity as claimed in Theorem~\ref{thm:sub-discounted-main}.

\end{document}